\documentclass{article}

\usepackage{microtype}
\usepackage{graphicx}
\usepackage{tikz}
\usepackage{booktabs} 
\usepackage{hyperref}
\usepackage{todonotes}
\usepackage{amsmath,amssymb,amsfonts,amsthm,bbm}
\usepackage{enumitem}
\usepackage{pifont}
\usepackage{csquotes}
\usetikzlibrary{arrows}
\usepackage{subcaption}
\usepackage{lipsum}  
\everypar{\looseness=-1}
\usepackage{verbatim}
\usepackage{comment}
\usepackage{float}
\usepackage{empheq}

\usepackage{hyperref}

\newcommand*\widefbox[1]{\fbox{\hspace{2em}#1\hspace{2em}}}

\newcommand{\indep}{\mathrel{\text{\scalebox{1.07}{$\perp\mkern-10mu\perp$}}}}
\newcommand\rsetminus{\mathbin{\mathpalette\rsetminusaux\relax}}
\newcommand\rsetminusaux[2]{\mspace{-4mu}
  \raisebox{\rsmraise{#1}\depth}{\rotatebox[origin=c]{-20}{$#1\smallsetminus$}}
 \mspace{-4mu}
}
\newcommand\rsmraise[1]{%
  \ifx#1\displaystyle .8\else
    \ifx#1\textstyle .8\else
      \ifx#1\scriptstyle .6\else
        .45%
      \fi
    \fi
  \fi}

\def\RR{{\mathbb R}}    
\def\PP{{\mathbb P}}     
\def\EE{{\mathbb E}}    
\def\11{{\mathbf 1}}    

  \def\cG{{\mathcal G}}     \def\cH{{\mathcal H}} \def\cN{{\mathcal N}}     \def\cO{{\mathcal O}}  \def\cD{{\mathcal D}}    \def\cV{{\mathcal V}}      \def\cF{{\mathcal F}}    \def\cX{{\mathcal X}} \def\cY{{\mathcal Y}}  \def\cZ{{\mathcal Z}}

                   \def\mfF{{\mathfrak F}}

         \def\bk{{\mathbf k}}             \def\bx{{\mathbf x}} \def\by{{\mathbf y}} \def\bz{{\mathbf z}} 

\def\bA{{\mathbf A}} \def\bB{{\mathbf B}}       \def\bI{{\mathbf I}} \def\bJ{{\mathbf J}} \def\bK{{\mathbf K}} \def\bL{{\mathbf L}}      \def\bR{{\mathbf R}}

\def\Tr{\operatorname{Tr}}

\def\Range{\operatorname{Range}}
\def\Ker{\operatorname{Ker}}
\def\Id{\operatorname{Id}}
\def\bk{\boldsymbol{k}}
\def\bell{\boldsymbol{\ell}}

\def\GP{\operatorname{GP}}
\DeclareSymbolFont{boldoperators}{OT1}{cmr}{bx}{n}
\SetSymbolFont{boldoperators}{bold}{OT1}{cmr}{bx}{n}

\newcommand\smallmath[2]{#1{\raisebox{\dimexpr \fontdimen 22 \textfont 2
      - \fontdimen 22 \scriptscriptfont 2 \relax}{\scalebox{.8}{$\scriptscriptstyle #2$}}}}

\newcommand\smallplus{\smallmath\mathbin +}
\def\brplus{\boldsymbol{r^{\smallplus}}}
\def\brplusx1{\brplus_{\!\!\!\!\bx_1}}
\def\bRplus{\bR^{\boldsymbol{\smallplus}}}

\theoremstyle{plain}
\newtheorem{theorem}{Theorem}[section]
\newtheorem{proposition}[theorem]{Proposition}
\newtheorem{lemma}[theorem]{Lemma}

\theoremstyle{definition}

\theoremstyle{remark}

\theoremstyle{plain}
    \newtheorem{innercustomgeneric}{\customgenericname}
    \providecommand{\customgenericname}{}
    \newcommand{\newcustomtheorem}[2]{%
      \newenvironment{#1}[1]
      {%
       \renewcommand\customgenericname{#2}%
       \renewcommand\theinnercustomgeneric{##1}%
       \innercustomgeneric
      }
      {\endinnercustomgeneric}
    }
    
    \newcustomtheorem{customthm}{Theorem}
    \newcustomtheorem{customlemma}{Lemma}
    \newcustomtheorem{customprop}{Proposition}

\usepackage[accepted]{icml2023}

\icmltitlerunning{Returning The Favour: When Regression Benefits From Probabilistic Causal Knowledge}

\begin{document}

\twocolumn[
\icmltitle{Returning The Favour: \\When Regression Benefits From Probabilistic Causal Knowledge}



\icmlsetsymbol{equal}{*}

\begin{icmlauthorlist}
\icmlauthor{Shahine Bouabid}{equal,stats}
\icmlauthor{Jake Fawkes}{equal,stats}
\icmlauthor{Dino Sejdinovic}{adelaide}
\end{icmlauthorlist}

\icmlaffiliation{stats}{Department of Statistics, University of Oxford, UK}
\icmlaffiliation{adelaide}{School of CMS \& AIML, University of Adelaide, Australia}

\icmlcorrespondingauthor{Shahine Bouabid}{shahine.bouabid@stats.ox.ac.uk}
\icmlcorrespondingauthor{Jake Fawkes}{jake.fawkes@stats.ox.ac.uk}

\icmlkeywords{Causality, Collider, Regression, Kernel Methods}

\vskip 0.3in
]



\printAffiliationsAndNotice{\icmlEqualContribution} 

\begin{abstract}
A directed acyclic graph (DAG) provides valuable prior knowledge that is often discarded in regression tasks in machine learning. We show that the independences arising from the presence of collider structures in DAGs provide meaningful inductive biases, which constrain the regression hypothesis space and improve predictive performance. We introduce \emph{collider regression}, a framework to incorporate probabilistic causal knowledge from a collider in a regression problem. When the hypothesis space is a reproducing kernel Hilbert space, we prove a strictly positive generalisation benefit under mild assumptions and provide closed-form estimators of the empirical risk minimiser. Experiments on synthetic and climate model data demonstrate performance gains of the proposed methodology.
\end{abstract}

\vspace*{-2em}
\section{Introduction}\label{section:intro}
\vspace*{-0.2em}

Causality has recently become a main pillar of research in the machine learning community. Historically, machine learning has been used to help solve problems in the field of causal inference~\citep{shalit2017estimating,zhang2012kernel}. But recently a different focus has emerged, asking what causality can do to return the favour to machine learning~\citep{scholkopf2021toward}. In this work we continue in this vein, and aim to answer whether the knowledge of a causal directed acyclic graph (DAG) underpinning the data generating process can assist and improve performance in regression tasks.

When a causal DAG is available, it constitutes a source of prior knowledge that is typically discarded when addressing a regression problem. It can however guide the setup of the regression problem. Classically, the structure of a DAG informs on which predictors should be selected to regress a given response variable $Y$. This process, known as feature selection, is solved by selecting the predictors that are either adjacent to $Y$, or that influence children of $Y$. The resulting set of predictors is called the Markov boundary of $Y$~\cite{pearl1987evidential}.

\begin{figure}[t!]
    \centering
    \includegraphics[width=0.65\linewidth]{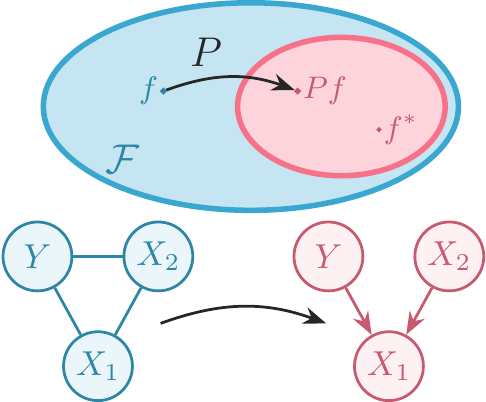}
    \vspace*{-1em}
    \caption{When performing regression in a hypothesis space $\cF$ (blue), we implicitly assume that the data generating process could follow any DAG structure. The optimal regressor $f^*$ lies in the subspace of function that satisfy the independence structure arising from the collider (pink), onto which the projection $P$ maps.}
    \label{fig:cover-figure}
    \vspace*{-1.4em}
\end{figure}

As we will see, the presence of a particular structure in a Markov boundary is typically overlooked in regression problems: colliders of the form $Y \rightarrow X_1 \leftarrow X_2$. 
In this work, we investigate how the conditional independence constraints arising due to colliders in the Markov boundary can be used to construct useful inductive biases in a regression problem and to guide the choice of the hypothesis space. We will see that the colliders are also unique in that regard: beyond colliders, the Markov boundary cannot contain any graphical structure implying a conditional independence with $Y$.

To understand the intuition behind colliders, consider this classic example: imagine we have a randomly timed sprinkler ($X_2$) and we want to infer whether it has rained ($Y$), having observed whether the sidewalk is wet ($X_1$). Although the sprinkler and the rain are marginally independent, knowing whether the sprinkler has been active is important for determining whether it has rained. Colliders arise naturally in many application domains. For example, in climate science, the objective may be to regress an environmental driver $Y$ that, independently from human activity $X_2$, influences observed global temperatures $X_1$. 


As illustrated in Figure~\ref{fig:cover-figure}, when performing least-square regression over a hypothesis space $\cF$, only a subset of $\cF$ will comply with the independences arising from the collider. By considering the projection operator $P$ that maps onto this subspace, we propose a framework called \emph{collider regression} to incorporate inductive biases arising from colliders into any regressor. We show that when the data generating process follows a collider, projecting any given regressor onto this subspace provides a positive generalisation benefit. 

We then consider the specific case where the hypothesis space is a reproducing kernel Hilbert space (RKHS). Because RKHSs are rich functional spaces that also enjoy closed analytical solutions to the least-squares regression problem, they allow us to build intuition for the general case. We prove a strictly  positive generalisation benefit from projecting the least-squares empirical risk minimiser in a RKHS, where the size of the generalisation gap increases with the complexity of the problem. We also show that for a RKHS, it is possible to solve the least-squares regression problem directly inside the projected hypothesis subspace and provide closed-form estimators. 

We experimentally validate the effectiveness of our methodology on a synthetic dataset and on a real world climate science dataset. Results demonstrate that collider regression consistently provides an improvement in generalisation at test time in comparison with standard least-squares regressors. Results also suggest that collider regression is particularly beneficial when few training samples are available, but samples from the covariates can easily be obtained, i.e. in a semi-supervised learning setting.

\vspace*{-0.8em}

\section{Background}\label{section:background}
\vspace*{-0.3em}
\paragraph{Regression notation}
Let $Y$ be our target variable over $\cY\subseteq\mathbb R$ and $X$ be our covariates over $\cX$. Our goal is a standard regression task where we have access to a dataset $\cD = \{\bx,  \by\}\in (\cX\times\cY)^n$ of $n$ samples $(x^{(i)}, y^{(i)})$ from $(X, Y)$. We aim to minimise the regularised empirical risk
\vspace*{-.1em}
\begin{equation}\label{eq:ERM}
    \hat f  = \underset{f\in\cF}{\arg\min}\, \frac{1}{n}\sum_{i=1}^n \left(y^{(i)} - f(x^{(i)})\right)^2 + \lambda \Omega(f)
\end{equation}
\vspace*{-.1em}
where $\cF$ is a specified hypothesis space of functions $f\!:~\!\!\cX\!\!\to~\!\!\!\cY$, $\lambda > 0$ and $\Omega(f) > 0$ is a regularisation term. This corresponds to finding a function $\hat f$ that best estimates the optimal regression function for the squared loss:
\begin{equation}
    f^*(x) = \EE[Y|X=x].
\end{equation}
For any two functions $h, h' \in \cF$, the squared-error generalisation gap between $h$ and $h'$ is defined as the difference in their true risk: 
\vspace*{-.1em}
\begin{equation}\label{eq:Gen gap}
    \Delta(h,h') = \EE[ \left( Y-h(X)\right)^2] - \EE[ \left( Y-h'(X)\right)^2].
\end{equation}
Therefore if $\Delta(h, h') \geq 0$, it means that $h'$ generalises better from the training data than $h$.
\vspace*{-1em}
\paragraph{Reproducing kernel Hilbert spaces}
Let $\cX$ be some non-empty space. A real-valued RKHS $(\cH, \langle \cdot, \cdot\rangle_\cH)$ is a complete inner product space of functions $f : \cX\to\RR$ that admits a bounded evaluation functional. For $x\in\cX$, the Riesz representer of the evaluation functional is denoted $k_x\in\cH$ and satisfies the \emph{reproducing property} $f(x) = \langle f, k_x\rangle_\cH$, $\forall f\in\cH$. The bivariate symmetric positive definite function defined by $k(x,x') = \langle k_x, k_{x'}\rangle_\cH$ is referred to as the \emph{reproducing kernel} of $\cH$. Conversely, the Moore-Aronszajn theorem~\cite{aronszajn1950theory} shows that any symmetric positive definite function $k$ is the unique reproducing kernel of an RKHS. For more details on RKHS theory, we refer the reader to \citet{berlinet2011reproducing}.
\vspace*{-1em}
\paragraph{Conditional Mean Embeddings}
Conditional mean embeddings (CMEs) provide a powerful framework to represent conditional distributions in a RKHS~\cite{fukumizu2004dimensionality, song2013kernel, muandet2016kernel}. Given random variables $X, Z$ on $\cX, \cZ$ and an RKHS $\cH\subseteq\RR^\cX$ with reproducing kernel $k : \cX\times\cX\to \RR$, the CME of $\PP(X|Z=z)$ is defined as 
\begin{equation}\label{eq:cme-definiton}
    \mu_{X|Z=z} = \EE[k_X|Z=z] \in \cH.
\end{equation}
It corresponds to the Riesz representer of the conditional expectation functional $f\mapsto \EE[f(X)|Z=z]$ and can thus be used to evaluate conditional expectations by taking an inner product $\EE[f(X)|Z=z] = \langle f, \mu_{X|Z=z}\rangle_\cH$.

Introducing a second RKHS $\cG\subseteq\RR^\cZ$ with reproducing kernel $\ell :\cZ\times\cZ\to\RR$, \citet{grunewalder2012conditional} propose an alternative view of CMEs as the solution to the least-square regression of canonical feature maps $\ell_Z$ onto $k_X$
\begin{equation}
    \left\{
    \begin{aligned}
        \begin{split}
            & E^* = \underset{C\in \mathsf{B}_2(\cG, \cH)}{\arg\min}\,\EE[\|k_X - C\ell_Z\|^2_\cH] \\
            & \mu_{X|Z=z} = E^*\ell_z
        \end{split}
    \end{aligned}
    \right.
\end{equation}
where $\mathsf{B}_2(\cG, \cH)$ denotes the space of Hilbert-Schmidt operators\footnote{i.e.\ bounded operators $A : \cG\to\cH$ such that $\Tr(A^*A) < \infty$. $\mathsf{B}_2(\cG, \cH)$ has a Hilbert space structure for the inner product $\langle A, B\rangle_{\mathsf{B}_2} = \Tr(A^*B)$.} from $\cG$ to $\cH$. Given a dataset $\cD = \{\bx, \bz\}$, this perspective allows to compute an estimate of the associated operator $E^*:\cG\to\cH$ as the solution to the regularised empirical least-squares problem as
\begin{equation}
    \!\!\left\{
    \begin{aligned}
        \begin{split}\label{eq:cmo-estimate}  
            & \hat E^* \!\! = \!\! \underset{C\in \mathsf{B}_2(\cG, \cH)}{\arg\min}\,\frac{1}{n}\sum_{i=1}^n\|k_{x^{(i)}}\! -\! C\ell_{z^{(i)}}\|_\cH^2 + \gamma \|C\|^2_{\mathsf{B}_2}\\
            & \quad = \bk_\bx^\top (\bL + \gamma\bI_n)^{-1}\bell_\bz \\
            & \hat \mu_{X|Z=z}  = \hat E^*\ell_z = \bk_\bx^\top (\bL + \gamma\bI_n)^{-1}\bell_\bz(z)
        \end{split}
    \end{aligned}
    \right.
\end{equation}
where $\gamma > 0$, $\bL = \ell(\bz, \bz)$, $\bk_\bx = k(\bx, \cdot)$ and $\bell_\bz = \ell(\bz, \cdot)$. We refer the reader to \cite{muandet2017kernel} for a comprehensive review of CMEs.

\section{DAG inductive biases for regression}\label{section:dag-inductive-regression}

In this section, we aim to answer how knowledge of the causal graph of the underlying data generating process can help to perform regression. We start by reviewing the concept of Markov boundaries and how it is used for feature selection. We then show that even after feature selection has been performed, there is still residual information from colliders that is relevant for a regression problem.

\subsection{Markov boundary for feature selection}

Since we are focusing on regression, we are interested in how the DAG can inform us about $\PP(Y|X)$. Suppose that for some vertex $X_i$, the DAG informs us that $Y\indep~X_i\mid~X\rsetminus X_i$. Stated in terms of mutual information we have that\footnote{This follows from $I(Y; \!X) = I (Y ;\! X \rsetminus X_i) + I (Y ;\! X_i |  X \rsetminus X_i)$ and the conditional independence gives $I (Y ; X_i |  X \rsetminus X_i) = 0$.}  $I(Y;X) = I (Y ; X \rsetminus X_i)$, therefore we can discard $X_i$ from our set of covariates without any loss of probabilistic information for $\PP(Y|X)$. 

From a functional perspective, we can interpret this as incorporating the inductive bias that the regressor need only depend on $X \rsetminus X_i$, allowing us to learn simpler functions which should generalise better from the training set.

By repeating the process of removing features, we can iteratively construct a minimal set of necessary covariates that still retain all the probabilistic information about $\PP(Y|X)$. This is known as feature selection~\cite{dash1997feature}.

Such a set, $S$, should satisfy $Y \indep X \rsetminus S | S$ and we should not be able to remove a vertex from $S$ without losing information about $\PP(Y|X)$. A set of this form is known as the Markov boundary of $Y$~\cite{statnikov2013algorithms}, denoted by $\operatorname{Mb}(Y)$. If the only independences in the distribution are those implied by the DAG structure\footnote{An assumption known as faithfulness~\cite{meek1995strong} which we take throughout.} then the Markov boundary is uniquely given by
\begin{align}\label{eq:Markov blanket}
 \operatorname{Mb}(Y) = \operatorname{Pa}(Y) \cup \operatorname{Ch}(Y) \cup \operatorname{Sp}(Y),   
\end{align}
where $\operatorname{Pa}(Y)$ are the parents of $Y$, $\operatorname{Ch}(Y)$ are the children of $Y$ and $\operatorname{Sp}(Y)$ are the spouses of $Y$, i.e.\ the children's other parents. In Figure~\ref{fig:Markov_boundary} the Markov boundary of $Y$ is highlighted in blue. 
\begin{figure}[t]
    \centering
\begin{tikzpicture}[>=stealth', shorten >=1pt, auto,
    node distance=1.5cm, scale=1.2, 
    transform shape, align=center, 
    state/.style={circle, draw, minimum size=7mm, inner sep=0.5mm}]
\node[state] (v0) at (0,0) {$Y$};
\node[state, above left of=v0,fill=blue!20] (v1) {$X_1$};
\node[state, above right of=v0,fill=blue!20] (v2) {$X_3$};
\node[state, below of=v0,fill=blue!20] (v3) {$X_6$};
\node[state, below right of=v0,fill=blue!20] (v4) {$X_5$};
\node[state, above of=v0,fill=red!20] (v5) {$X_2$};
\node[state, below left of=v0,fill=blue!20,yshift=0.6cm] (v6) {$X_4$};
\node[state, below of=v6,,fill=red!20,yshift=0.4cm] (v7) {$X_7$};
\draw [->, thick] (v2) edge (v0);
\draw [->, thick] (v1) edge (v0);
\draw [->, thick] (v5) edge (v1);
\draw [->, thick] (v5) edge (v2);
\draw [->, thick] (v0) edge (v3);
\draw [->, thick] (v6) edge (v3);
\draw [->, thick] (v4) edge (v3);
\draw [->, thick] (v2) edge (v4);
\draw [->, thick] (v3) edge (v7);
\end{tikzpicture} 
\caption{A causal graph with the Markov boundary of $Y$ highlighted in blue and vertices outside the Markov boundary highlighted in red. Whilst $Y$ and $X_4$ are marginally independent, the presence of the collider $X_6$ opens the path between $Y$ and $X_4$.}
\label{fig:Markov_boundary}
\end{figure}
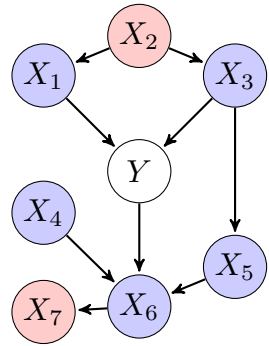

\subsection{Extracting inductive bias for regression}

By construction the Markov boundary of $Y$ cannot contain independence relationships of the form $Y\indep~X_i|X\rsetminus X_i$. However, it can still contain unused independence statements that involve $Y$, and therefore provides useful information about the conditional distribution $\PP(Y|X)$.


For example, the graphical structure in Figure~\ref{fig:Markov_boundary} gives that $Y \indep X_4$ and $Y \indep X_5 \mid X_3$. This implies that $\PP(Y |X_4) = \PP(Y)$ and $\PP(Y |X_3, X_5) = \PP(Y |X_3)$ which by marginalisation constrains $\PP(Y |X)$ and so gives us extra information about it. The presence of these independence relationships inside $\operatorname{Mb}(Y)$ is only possible because a collider, $X_6$, has allowed for the spouses $X_4$ and $X_5$ to be within the Markov boundary without being adjacent to $Y$.

Hence, the presence of collider structures within the Markov boundary of $Y$ provides additional independence relationships involving $Y$. The following proposition shows that the presence of a collider is not only a sufficient condition, but also necessary.

\begin{proposition}\label{prop:only_colliders}
The Markov boundary of $Y$ contains a collider if and only if there exists $Z \in\operatorname{Mb}(Y)$ and $S_Z\subset~\operatorname{Mb}(Y)$ such that  $Y \indep Z \mid S_Z$.
\end{proposition}
\begin{proof}
We have a conditional independence between two variables if and only if they are not adjacent~(Lemma 3.1, 3.2 \citet{koller2009probabilistic}) and $\operatorname{Mb}(Y)$ contains a variable not adjacent to $Y$ if and only if it contains a collider.
\end{proof}


The collider structures are thus the only graphical structures that provide conditional independence statement relevant to $\PP(Y|X)$ within the Markov boundary. To the best of our knowledge, this information is currently left unused when addressing a regression problem.

However, unlike for the feature selection process, we cannot simply use these independence statements to discard covariates and reduce the set of features. This is because while the spouses of $Y$ are uninformative on their own, they become informative in the presence of other covariates. Namely in Figure~\ref{fig:Markov_boundary}, while $Y \indep X_4$ we have $Y \not \indep X_4 | X_6$ because $X_6$ is a collider. Therefore, we have that $I(Y;X) > I (Y ; X \rsetminus X_4)$ and discarding $X_4$ would constitute a loss of information.


\section{Collider Regression}\label{section:collider-regression}

In this section, we present a method for incorporating probabilistic inductive bias from a collider structure into a regression problem, and provide guarantees of improved generalisation error. For the sake of clarity, our exposition focuses on the simple collider structure depicted in Figure~\ref{fig:simple-collider}. We however emphasise this simplification does not harm the generality of our contribution and Section~\ref{section:more-general-dag} shows how collider regression can be extended to more general DAGs.

\begin{figure}[H]
     \centering
         \begin{tikzpicture}[>=stealth', shorten >=1pt,
                             node distance=1.5cm, scale=1.05, 
                             transform shape, align=center, 
                             state/.style={circle, draw, minimum size=7mm, inner sep=0.5mm}]
            \node[state] (v2) at (0,0) {$X_1$};
            \node[state, above right = -0.5 and 0.8 of v2] (v0) {$X_2$};
            \node[state, above left = -0.5 and 0.8 of v2] (v1) {$Y$};
            \draw [->, thick] (v0) edge (v2);
            \draw [->, thick] (v1) edge (v2);
        \end{tikzpicture}
    \caption{Simple collider structure}
    \label{fig:simple-collider}
\end{figure}
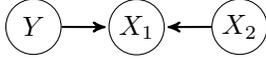

\subsection{Simple collider regression setup}\label{subsection:setup}
Let $X_1, X_2, Y$ be random variables following the DAG structure in Figure \ref{fig:simple-collider} and taking values in $\cX_1\subseteq\RR^{d_1}$, $\cX_2\subseteq\RR^{d_2}$ and $\cY\subseteq\RR$ respectively. Without loss of generality, we assume that $\EE[Y] = 0$.

Under the squared loss, the optimal regressor is given by
\begin{equation}
    f^*(x_1, x_2) = \EE[Y|X_1 =x_1, X_2=x_2].
\end{equation}
Since the collider gives the independence relationship $Y~\indep~X_2$, we have that
\begin{align}
\begin{split}
    \EE[f^*(X_1, X_2)|X_2] & = \EE\big[\EE[Y|X_1, X_2]\mid X_2\big] \\
                            & = \EE[Y | X_2 ] \\
                            & = \EE[Y] \\
                            & = 0,
\end{split}
\end{align}
where the second line comes from the tower property of the conditional expectation.

Hence, the optimal regressor $f^*$ lies in the subspace of functions that have zero $X_2$-conditional expectation. To incorporate the knowledge from the DAG into our regression procedure, we should therefore ensure that our estimate $\hat f$ lies within the same subspace of functions, i.e.\ we want to satisfy the zero conditional expectation constraint
\begin{equation}\label{eq:constraint-on-f}
\tag{ZCE}
    \hat f \in\big\{f\in\cF\mid \EE[f(X_1, X_2)|X_2] = 0\big\}.
\end{equation} 
We propose to investigate how such a constraint can be enforced onto our hypothesis and how it benefits generalisation, starting by the general case of square-integrable functions. In what follows, we will use shorthand concatenated notations $X = (X_1, X_2)$, $\cX = \cX_1\times\cX_2$, $x = (x_1, x_2)\in\cX$ and $\bx = (\bx_1, \bx_2)\in\cX^n$.

\subsection{Respecting the collider structure in the hypothesis}\label{subsection:L2-case}
Let $L^2(X)$ denote the space of square-integrable functions with respect to the probability measure induced by $X$ and suppose $\cF = L^2(X)$. Let $E : L^2(X) \to L^2(X)$ denote the conditional expectation operator defined by
\begin{equation}
    Ef(x_1, x_2) = \EE[f(X_1, X_2)|X_2 = \pi_2(x_1, x_2)],
\end{equation}
where $\pi_2(x_1, x_2) = x_2$ is simply the mapping that discards the first component\footnotemark.

The operator $E$ classically defines an orthogonal projection over the subspace of $X_2$-measurable functions. $L^2(X)$ thus orthogonally decomposes into its image, denoted $\Range(E)$, and its null-space, denoted $\Ker(E)$, as
\begin{equation}\label{eq:L2-decomposition}
    L^2(X) = \Ker(E) \oplus \Range(E).
\end{equation}

\footnotetext{This notation emphasises that $Ef$ is formally a function of $(x_1, x_2)$ and belongs in $L^2(X)$}

Using this notation, satisfying condition (\ref{eq:constraint-on-f}) corresponds to having $\hat f\in\Ker(E)$. Alternatively, if we denote
\begin{equation}\label{eq:l2-p-expression}
    P = \Id - E,
\end{equation}
the orthogonal projection onto $\Ker(E)$, then we want to take $\cF = \Range(P)$ as our hypothesis space.

In general, it may be hard to constrain the hypothesis space directly to be $\Range(P)$. However, the solution to the empirical risk minimisation problem (\ref{eq:ERM}) will always orthogonally decompose within $L^2(X)$ as
\begin{equation}
    \hat f = P\hat f + E\hat f,
\end{equation}
where only $P\hat f\in\Range(P)$ satisfies (\ref{eq:constraint-on-f}). It turns out that discarding $E \hat f$ --- the part that does not satisfy the constraint --- will always yield generalisation benefits. 

\begin{proposition}\label{proposition:L2-generalisation-benefit}
    Let $h \in L^2(X)$ be any regressor from our hypothesis space. We have
    \begin{equation}
        \Delta(h, Ph)  = \|Eh\|_{L^2(X)}^2.
    \end{equation}
\end{proposition}

The generalisation gap is always greater than zero. Hence, for any given regressor $\hat f$, we can always improve its test performance by projecting it onto $\Range (P)$. 

In practice, a simple estimator of $P\hat f$ can be obtained by subtracting an estimate of $\EE[\hat f(X_1, X_2)|X_2]$ as
\begin{equation}
    \hat P \hat f(x_1, x_2) = \hat f(x_1, x_2) - \hat\EE[\hat f(X_1, X_2)|X_2=x_2]
\end{equation}
by following the procedure outlined in Algorithm~\ref{alg:L2-collider-regression}.

\begin{algorithm}[h]
    \caption{General procedure to estimate $P \hat f$}
    \label{alg:L2-collider-regression}
\begin{algorithmic}[1]
    \STATE Regress $(X_1, X_2) \rightarrow Y$ \\ to get  $(x_1, x_2)\mapsto \hat f(x_1, x_2)$
    \STATE Regress $X_2 \rightarrow \hat f(X_1, X_2)$  \\ to get $x_2\mapsto \hat\EE[\hat f(X_1, X_2)|X_2=x_2]$
    \STATE Take $\hat P \hat f(x_1, x_2) \!=\! \hat f(x_1, x_2) - \hat\EE[\hat f(X_1, X_2)|X_2\!=\!x_2]$
\end{algorithmic}
\end{algorithm}

It is worth noting that the second step of Algorithm~\ref{alg:L2-collider-regression} does not require observations from $Y$. As such, it naturally fits a semi-supervised setup where additional observations $\cD' = \{\bx_1', \bx_2'\}$ are available, and can be used to produce a better estimate of the conditional expectation $\EE[\hat f(X_1, X_2)|X_2]$.


\subsection{Theoretical guarantees in a RKHS}\label{subsection:rkhs-case}

RKHSs are mathematically convenient functional spaces and under mild assumptions on the reproducing kernel, they can be proven to be dense in $L^2(X)$~\cite{sriperumbudur2011universality}. This makes them a powerful tool for theoretical analysis and building intuition which can be expected to carry over to more general function spaces. For this reason, in this section we study the case where the hypothesis space is a RKHS $\cF = \cH$. We denote its inner product by $\langle\cdot, \cdot\rangle_\cH$ and its reproducing kernel $k : \cX\times\cX\to\RR$.


When solving the least-square regression problem in a RKHS, it is known that for Tikhonov regularisation $\Omega(f) = \| f\|_\cH^2$, the solution to the empirical risk minimisation problem (\ref{eq:ERM}) in $\cH$ enjoys a closed-form expression given by
\begin{equation}\label{eq:krr}
    \hat f = \by^\top \left(\bK + \lambda \bI_n\right)^{-1}\bk_{\bx},
\end{equation}
where $\bK = k(\bx, \bx)$ and $\bk_\bx = k(\bx, \cdot)$.

Therefore, if we now project $\hat f$ onto $\Range(P)$ as previously, the projected empirical risk minimiser writes
\begin{equation}\label{eq:projected-erm}
    P \hat f = \by^\top \left(\bK + \lambda \bI_n\right)^{-1}P\bk_{\bx}
\end{equation}
with notation abuse $P\bk_{\bx} = [Pk_{x^{(1)}} \hdots Pk_{x^{(n)}}]^\top$.

Leveraging these analytical expressions, the following result establishes a strictly non-zero generalisation benefit from projecting $\hat f$. The proof techniques follows that of \citet{elesedy2021provably}, but is adapted to our particular setup with relaxing assumptions about the projection orthogonality\footnote{$P$ is not necessarily orthogonal anymore as a projection of $\cH$} and the form of the data generating process.
\begin{theorem}\label{thm:main}
    Suppose $M = \sup_{x\in\cX} k(x, x) < \infty$ and $\operatorname{Var}(Y|X) \geq \eta > 0$. Then, the generalisation gap between $\hat f$ and $P\hat f$ satisfies
    \begin{equation}
         \EE[\Delta(\hat f, P\hat f)] \geq \frac{\eta \EE\big[\|\mu_{X|X_2}(X)\|_{L^2(X)}^2\big]}{\left(\sqrt{n}M + \lambda / \sqrt{n}\right)^2}
    \end{equation}
    where $\mu_{X|X_2} = \EE[k_X|X_2]$ is the CME of $\PP(X|X_2)$.
\end{theorem}

\setcounter{footnote}{5}
\footnotetext{$E$ then corresponds to what is referred to as a conditional mean operator in the kernel literature~\cite{fukumizu2004dimensionality}.}
\setcounter{footnote}{4}

This demonstrates that in a RKHS, projecting the empirical risk minimiser is strictly beneficial in terms of generalisation error. Specifically, if there exists a set with non-zero measure on which $Y \neq 0$ and $\mu_{X|X_2} \neq 0$ almost-everywhere, then the lower bound is strictly positive. 

The magnitude of the lower bound depends on the variance of $\|\mu_{X|X_2}(X)\|_{L^2(X)}$ and the lower bound on $\operatorname{Var}(Y|X)$. This indicates that problems with more complex conditional distributions $\PP(X|X_2)$ and $\PP(Y|X)$ should enjoy a larger generalisation gap.


The theorem also suggests that the lower bound on the generalisation benefit decreases at the rate $\cO(1/n)$ as the number of samples $n$ grows. Since for the well-specified kernel ridge regression problem, the excess risk upper bound also decreases at rate $\cO(1/n)$~\cite{bach2021learning, caponnetto2007optimal}, we have that $\EE[\Delta(\hat f, P\hat f)] = \Theta (1/n)$.

In a RKHS, $P\hat f$ can be rewritten using CMEs as
\begin{equation}
    Pf(x_1, x_2) = f(x_1, x_2) - \langle f, \mu_{X|X_2=x_2}\rangle_\cH.
\end{equation}
Therefore, introducing a kernel $\ell:\cX_2\times\cX_2\to\RR$, the CME estimate from (\ref{eq:cmo-estimate}) allows to devise an estimator of $P\hat f$ as:
\begin{equation}\label{eq:RKHS-pf-estimate}
\hat P \hat f \!=\!  \by^\top \!\left(\bK \!+\! \lambda \bI_n\right)^{-1}\!\left(\bk_\bx\! - \!\bK(\bL \!+\! \gamma \bI_n)^{-1}\!\bell_{\bx_2}\right) 
\end{equation}
where $\bL = \ell(\bx_2, \bx_2)$, $\bell_{\bx_2} = \ell(\bx_2, \cdot)$ and $\gamma > 0$.

\subsection{Respecting the collider structure in a RKHS}\label{subsection:new-rkhs-case}

Similarly to the $L^2(X)$ case, the solution to the empirical risk minimisation problem in $\cH$ will also decompose as $\hat f = P\hat f + E\hat f$. Thus, we can proceed similarly by simply discarding $E\hat f$ to improve performance. However, it turns out that using elegant functional properties of RKHSs, it is possible to take a step further and directly take $\cF = \Range(P)$. In doing so, we can ensure that our hypothesis space only contains functions that satisfy constraint (\ref{eq:constraint-on-f}).

Under assumptions detailed in Appendix~\ref{appendix:rkhs-assumptions}, we can view the projection $P$ as a well-defined RKHS projection\footnotemark\ $P~:~\cH~\to~\cH$. In particular, an important assumption is that the kernel takes the form
\begin{equation}\label{eq:kernel_assumption}        k\left(x,x'\right)=\left(r\left(x_{1},x_{1}'\right)+1\right)\ell\left(x_{2},x_{2}'\right),
\end{equation}
where $r : \cX_1\times\cX_1\to\RR$ and $\ell:\cX_2\times\cX_2\to\RR$ are also positive semi-definite kernels. This ensures that $\cH$ contains functions that are constant with respect to $x_1$. Thus, the conditional expectation mapping $(x_1, x_2)\mapsto\EE[f(X_1, X_2)|X_2~=~x_2]$ belongs to the same RKHS.

If these assumptions are met, we denote $\cH_P = \Range(P)$. The following result characterises $\cH_P$ as a RKHS.
\begin{proposition}\label{proposition:projected-RKHS}
    Let $P^*$ be the adjoint operator of $P$ in $\cH$. Then $\cH_P$ is also a RKHS with reproducing kernel
    \begin{equation}
        k_P(x,x') = \langle P^* k_x, P^* k_{x'}\rangle_\cH
    \end{equation}
     with $P^* k_x = k_x - \mu_{X|X_2=\pi_2(x)}$.
\end{proposition}

Using the projected RKHS kernel $k_P$, it becomes possible to solve the least-square regression problem directly inside $\cF = \cH_P$. By taking $\Omega(f) = \|f\|_{\cH_P}^2$, the empirical risk minimisation problem becomes a standard kernel ridge regression problem in $\cH_P$ which admits closed-form solution
\begin{equation}\label{eq:erm-in-HP}
    \hat f_P = \by^\top \left(\bK_P + \lambda \bI_n\right)^{-1}\bk_{P,\bx},
\end{equation}
where $\bK_P = k_P(\bx, \bx)$ and $\bk_{P,\bx} = k_P(\bx, \cdot)$.

From a learning theory perspective, performing empirical risk minimisation inside $\cH_P$ should provide tighter bounds on the generalisation error than on the entire space $\cH$. This is because since $\cH_P\subset\cH$, the Rademacher complexity of $\cH_P$ is smaller than that of $\cH$.

It should be noted that $k_P$ depends on the CME $\mu_{X|X_2=\pi_2(x)}$, which needs to be estimated. Therefore, in practice, our hypothesis will not lie in the true $\cH_P$ but in an approximation of $\cH_P$ and the approximation error will depend directly on the CME estimation error.

\begin{algorithm}[h]
    \caption{RKHS procedure to estimate $\hat f_P$}
    \label{alg:RKHS-collider-regression}
\begin{algorithmic}[1]
    \STATE Let $\hat P^*k_x = k_x - \hat \mu_{X|X_2=\pi_2(x)}$
    \STATE Let $\hat k_P(x, x') = \langle \hat P^*k_x, \hat P^*k_{x'}\rangle_\cH$
    \STATE Evaluate $\hat \bK_P = \hat k_P(\bx, \bx)$ and $\hat\bk_{P,\bx} = \hat k_P(\bx, \cdot)$
    \STATE Take $\hat f_{\hat P} = \by^\top (\hat \bK_P + \lambda \bI_n)^{-1}\hat\bk_{P,\bx}$
\end{algorithmic}
\end{algorithm}

The estimation of (\ref{eq:erm-in-HP}) is again a two-stage procedure outlined in Algorithm~\ref{alg:RKHS-collider-regression}. The distinction with the general $L^2(X)$ case is that we do not estimate the conditional expectation of any specific function. Instead, we estimate the conditional expectation operator through $\hat\mu_{X|X_2=x_2}$, and then use it through $\hat P^*$ to constrain the hypothesis space. This is possible because in a RKHS, the estimation of the conditional expectation operator can be achieved independently from the function it is applied to. Due to the assumption on the kernel introduced in equation \ref{eq:kernel_assumption} there are now alternative estimators for $\hat\mu_{X|X_2=x_2}$ which we provide details of 
in Appendix~\ref{appendix:estimators-details}.

The estimation of $P^*$ in line 1 only requires observations from $X_1, X_2$. Thus, like in the $L^2(X)$ case, additional observations $\cD' = \{\bx_1', \bx_2'\}$ can help better estimate CMEs, and thus better approximate the projected RKHS $\cH_P$.

\begin{figure*}[t]
    \centering
    \includegraphics[width=\textwidth]{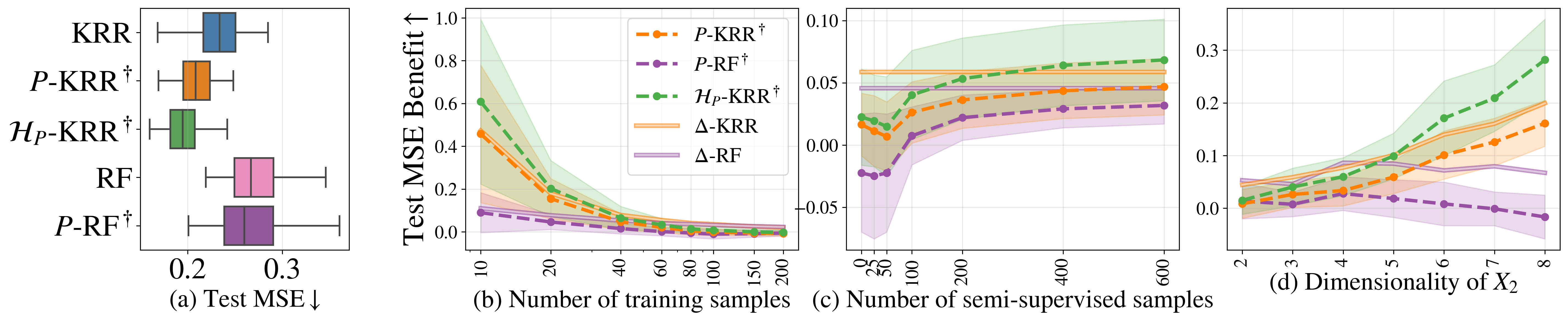}
    \caption{(a) :  Test MSEs for the simulation experiment ; dataset is generated using $d_1 = 3$, $d_2 = 3$, $n = 50$ and $100$ semi-supervised samples ; experiments is run for 100 datasets generated with different seeds ; statistical significance is confirmed in Appendix~\ref{appendix:experiments} ; (b, c, d) : Ablation study on the number of training samples, number of semi-supervised samples and dimensionality of $X_2$ ; experiments are run for 40 datasets generated with different seeds ; $\uparrow\!/\!\downarrow$ indicates higher/lower is better ; we report 1 s.d. ; $\dagger$ indicates our proposed methods.}
    \label{fig:simulation-results}
    \vspace*{-0.4em}
\end{figure*}

\section{Collider Regression on a more general DAG}\label{section:more-general-dag}

We now return to a general Markov boundary. Any Markov boundary may be partitioned following Figure~\ref{fig:general-collider}, where $X_1$ contains all direct children of $Y$, $X_3$ contains all parents of $Y$ and all other variables are grouped in $X_2$. Furthermore, we assume that there exists no edge from a variable in $X_1$ to a variable in $X_2$.

This provides us with the probabilistic information that $Y \indep X_2 \mid X_3$ but $Y \not \indep X_2 \mid X_3, X_1$, which implies in expectation that $\EE[Y|X_3] = \EE[Y|X_2, X_3]$.

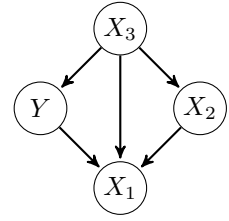
\begin{figure}[H]
     \centering
\begin{tikzpicture}[>=stealth', shorten >=1pt,
                             node distance=1.5cm, scale=1, 
                             transform shape, align=center, 
                             state/.style={circle, draw, minimum size=7mm, inner sep=0.5mm}]
            \node[state] (v2) at (0,0) {$X_1$};
            \node[state, above right of=v2] (v0) {$X_2$};
            \node[state, above left of=v2] (v1) {$Y$};
            \node[state, above right of=v1] (v3) {$X_3$};
            \draw [->, thick] (v0) edge (v2);
            \draw [->, thick] (v1) edge (v2);
            \draw [->, thick] (v3) edge (v1);
            \draw [->, thick] (v3) edge (v2);
            \draw [->, thick] (v3) edge (v0);
        \end{tikzpicture} 
    \caption{General Markov boundary collider structure.}
    \label{fig:general-collider}
\end{figure}

If we now denote $X = (X_1, X_2, X_3)$ and $f_0(x) = \EE[Y|X_3=x_3]$, then the optimal least-square regressor $f^*(x)=\EE[Y|X=x]$ satisfies
\begin{align}
    \begin{split}
        & \EE\big[f^*(X) - f_0(X)\mid X_2, X_3\big] \\
        = &\, \EE\left[\EE[Y|X]\mid X_2, X_3\right] - \EE\left[\EE[Y|X_3]\mid X_2, X_3\right] \\
        = &\, \EE[Y|X_2, X_3] - \EE\left[\EE[Y|X_2, X_3]\mid X_2, X_3\right]\\
        = &\, 0.
    \end{split}
\end{align}
Therefore, if we center our hypothesis space on $f_0$, then like in Section~\ref{subsection:setup}, we want our centered estimate $\hat{f} - f_0$ to lie within the following subspace:
\begin{equation}
    \hat{f} - f_0 \in \big\{f \in \cF \mid \EE\big[f(X)\mid X_2, X_3\big] = 0\big\}.
\end{equation}
When $\cF = L^2(X)$, this space can again be seen as the range of an orthogonal projection, this time defined by
\begin{equation}
    P' = \Id - E'
\end{equation}
where $E' : L^2(X)\to L^2(X)$ denotes the conditional expectation functional with respect to $(X_2, X_3)$
\begin{equation}
    E'f(x_2, x_3) = \EE[f(X)|X_2 = x_2, X_3=x_3].
\end{equation}


While we focus in Section~\ref{section:collider-regression} on the simple collider structure for the sake of exposition, our result are stated for a general projection operator and still hold for $P'$ --- modulo a shift by $f_0$. Hence, we can still apply the techniques we have presented to encode probabilistic information from the general DAG in Figure~\ref{fig:general-collider} into a regression problem, with similar guarantees on the generalisation benefits.
\begin{proposition}\label{proposition:general-L2-generalisation-benefit}
    Let $h\in L^2(X)$ be any regressor from our hypothesis space. We have
    \begin{equation}
        \Delta(h, f_0 + P'h) = \|E'h - f_0\|_{L^2(X)}^2.
    \end{equation}
\end{proposition}
This means that, for any given regressor $\hat f$, we can always improve its test performance by first projecting it onto $\Range(P)$, and then shifting it by $f_0$. 

In practice, the estimation strategies introduced in Section~\ref{section:collider-regression} can still be applied to obtain an estimate of $P'\hat f$. An additional procedure to estimate $f_0$ will however be needed. This can be achieved by regressing $Y$ onto $X_3$. We provide corresponding algorithms and estimators in Appendix~\ref{appendix:general-case-details}.




\section{Experiments}\label{section:experiments}

This section provides empirical evidence that incorporating probabilistic causal knowledge into a regression problem benefits performance. First, we demonstrate our method on an illustrative simulation example. We conduct an ablation study on the number of training samples, the dimensionality of $X_2$ and the use of additional semi-supervised samples. Then, we address a challenging climate science problem that respects the collider structure. Our results underline the benefit of enforcing constraint (\ref{eq:constraint-on-f}) onto the hypothesis. Code and data are made available\footnote{\url{https://github.com/shahineb/collider-regression}.}.

\paragraph{Models}
We compare five models: 
\begin{enumerate}
    \setlength\itemsep{-0.1em}
    \item \emph{RF}: A baseline random forest model.
    \item \emph{$P$-RF}: The baseline RF model projected following Algorithm~\ref{alg:L2-collider-regression} and using a linear regression to estimate $\hat\EE[\hat f(X_1, X_2)|X_2=x_2]$.
    \item \emph{KRR}: A baseline kernel ridge regression. 
    \item \emph{$P$-KRR}: The KRR model projected following (\ref{eq:RKHS-pf-estimate}).
    \item \emph{$\cH_P$-KRR}:  A kernel ridge regression model fitted directly in the projected RKHS following Algorithm~\ref{alg:RKHS-collider-regression}.
\end{enumerate}

For both KRR and RF, we use Proposition~\ref{proposition:L2-generalisation-benefit} to compute Monte Carlo estimates of the expected generalisation gap $\EE[\Delta(\hat{f},P\hat{f})]$, which we denote as $\Delta$-KRR and $\Delta$-RF respectively. This provides an indicator of the greatest achievable generalisation gain if we had access to the exact projection $P$. Hyperparameters are tuned using a cross-validated grid search and model details are specified in Appendix~\ref{appendix:experiments}.

\subsection{Simulation example}\label{subsection:simulation-example}
\paragraph{Data generating process}
We propose the following construction that follows the simple collider structure from Figure~\ref{fig:simple-collider}. Let $d_1, d_2 \geq 1$ denote respectively the dimensionalities of $X_1$ and $X_2$. We first generate a fixed positive definite matrix $\Sigma$ of size $(d_1 + d_2 + 1)$ which has zero off-diagonals on the $(d_1 + d_2)$\textsuperscript{th} row and column . We then follow the generating process described in Algorithm~\ref{alg:simulation-example} and generate a dataset of $n$ observations $\cD = \{\bx_1, \bx_2, \by\}$. The zero off-diagonal terms in $\Sigma$ ensure that we satisfy $Y\indep X_2$ and  $g_1, g_2$ are nontrivial mappings that introduce a non-linear dependence (details in Appendix~\ref{appendix:experiments}).
\begin{algorithm}[h]
    \caption{Data generating process simulation example}
    \label{alg:simulation-example}
\begin{algorithmic}[1]
    \STATE {\bfseries Input:} $\Sigma\succcurlyeq 0$, $\sigma > 0$, $g_1\!:\! \RR^{d_1}\!\to\!\RR^{d_1}$, $g_2:\RR^{d_2}\!\to\!\RR^{d_2}$
    \STATE $\begin{bmatrix}X_1 & X_2 & Y\end{bmatrix}^\top \sim \cN(0, \Sigma)$,  $\enspace\varepsilon\sim\cN(0, \sigma^2)$
    \STATE $X_1 \gets g_1(X_1) + \varepsilon$
    \STATE $X_2 \gets g_2(X_2)$
    \STATE \textbf{return} $X_1, X_2, Y$
\end{algorithmic}
\end{algorithm}

\paragraph{Results}
Figure~\ref{fig:simulation-results}(a) provides empirical evidence that, for both KRR and RF, incorporating probabilistic inductive biases from the collider structure in the hypothesis benefits the generalisation error.

In addition, Figure~\ref{fig:simulation-results}(b)(c)(d) shows that the empirical generalisation benefit is greatest when : fewer training samples are available, semi-supervised samples can be easily obtained and the dimensionality of $X_2$ is larger. This is in keeping with Theorem~\ref{thm:main} which predicts the benefit will be larger when we have fewer labeled samples and a more complicated relationship between $X_2$ and $X_1$.

Because the decision nodes learnt by RF largely rely on $X_1$ and the early dimensions of $X_2$, increasing the dimensionality of $X_2$ has little to negative effect as shown in Figure~\ref{fig:simulation-results}(d).




\subsection{Aerosols radiative forcing}\label{subsection:aerosol-example}

\paragraph{Background}
The radiative forcing is defined as the difference between incoming and outgoing flux of energy in the Earth system. At equilibirum, the radiative forcing should be of 0$\,$W$\,$m\textsuperscript{-2}. Carbon dioxide emissions from human activity contribute a positive radiative forcing of +1.89$\,$W$\,$m\textsuperscript{-2} which causes warming of the Earth~\cite{bellouin2020radiative}.

Aerosols are microscopic particles suspended in the atmosphere (e.g.\ dust, sea salt, black carbon) that contribute a negative radiative forcing by helping reflect solar radiation, which cools the Earth. However, the magnitude of their forcing represents the largest uncertainty in assessments of global warming, with uncertainty bounds that could offset global warming or double its effects. It is thus critical to obtain better estimate of the aerosol radiative forcing.

The carbon dioxide and aerosol forcings are independent factors\footnote{this is because whilst human activity can confound CO\textsubscript{2} and aerosol emissions, the timescale on which CO\textsubscript{2} and aerosol forcing operate (century vs week) are so different that the forcings at a given time can be considered independent.} that contribute to the observed global temperatures. Hence, by setting $Y=$ \enquote{aerosol forcing}, $X_2=$\enquote{CO\textsubscript{2} forcing} and $X_1=$ \enquote{global temperature}, this problem has a collider structure and observations from global temperature and CO\textsubscript{2} forcing can be used to regress the aerosol forcing.



\paragraph{Data generating process}

FaIR (for Finite amplitude Impulse Response) is a deterministic model that proposes a simplified low-order representation of the climate system~\cite{millar2017modified, smith2018fair}. Surrogate climate models like FaIR --- referred to as \emph{emulators} --- have been widely used, notably in reports of the Intergovernmental Panel on Climate Change~\cite{masson2021climate}, because they are fast and inexpensive to compute. 

We use a modified version of FaIRv2.0.0~\cite{leach2021fairv2} where we introduce variability by adding white noise on the forcing to account for climate internal variability~\cite{hasselman1976stochastic, cummins2020optimal}. To generate a sample, we run the emulator over historical greenhouse gas and aerosol emission data and retain scalar values for $y =$ \enquote{aerosol forcing in 2020}, $x_2 = $ \enquote{CO\textsubscript{2} forcing in 2020} and $x_1 = $ \enquote{global temperature anomaly in 2020}. We perform this $n$ times to generate dataset $\cD = \{\bx_1, \bx_2, \by\}$.

\begin{table}[t]
    \centering
    \caption{MSE, signal-to-noise ratio (SNR) and correlation on test data for the aerosol radiative forcing experiment ; $n = 50$ and $200$ semi-supervised samples ; statistical significance is confirmed in Appendix~\ref{appendix:experiments} ; experiments is run for 100 datasets generated with different seeds ; $\uparrow\!/\!\downarrow$ indicates higher/lower is better ; we report 1 standard deviation ; $\dagger$ indicates our proposed methods.}
    \begin{tabular}{rccc}\toprule
           & MSE\small{$\;\downarrow$} & SNR\small{$\;\uparrow$} & Correlation\small{$\;\uparrow$} \\ \midrule
    RF     &  0.90\scriptsize{$\pm$0.04}  &  0.44\scriptsize{$\pm$0.19}   &   0.32\scriptsize{$\pm$0.08}    \\
    $P$-RF$^{\dagger}$   &  \textbf{0.89\scriptsize{$\pm$0.03}}  &   \textbf{0.49\scriptsize{$\pm$0.15}}  &   \textbf{0.34\scriptsize{$\pm$0.07}}    \\ \midrule
    KRR    &  0.88\scriptsize{$\pm$0.04}   &  0.58\scriptsize{$\pm$0.17}   &  0.37\tiny{$\pm$0.05}    \\
    $P$-KRR$^{\dagger}$  &  \textbf{0.86\scriptsize{$\pm$0.03}}   &  \textbf{0.65\scriptsize{$\pm$0.13}}   &  \textbf{0.40\tiny{$\pm$0.01}}    \\
    $\cH_P$-KRR$^{\dagger}$ &  \textbf{0.86\scriptsize{$\pm$0.03}}   &   \textbf{0.65\scriptsize{$\pm$0.14} } &  \textbf{0.40\scriptsize{$\pm$0.01}}    \\ \bottomrule
    \end{tabular}
    \label{table:FaIR-results}
\end{table}

\paragraph{Results}

Results are reported in Table~\ref{table:FaIR-results}. We observe that the incorporation of inductive bias from the collider resulted in consistently improved performance for both RF and KRRs. This shows that while the proposed methodology is only formulated in terms of squared error, it can also improve performance for other metrics.

\section{Discussion and Related Work}\label{section:discussion}
\paragraph{Regression and Causal Inference}

Currently causal inference is most commonly used in regression problems when reasoning about invariance~\cite{peters2016causal,arjovsky2019invariant}. These methods aim to use the causal structure to guarantee the predictors will transfer to new environments~\cite{gulrajani2020search} and recent work discusses how causal structure plays a role in the effectiveness of these methods~\cite{wang2022a}. Our work takes a complimentary route in asking how causal structure can benefit in regression, and, in contrast to prior work, focuses on a fixed environment. 

\paragraph{Causal and Anti-causal learning} 

Our work is closely related to work on anti-causal learning~\cite{scholkopf2012causal} which argues that $\PP(X)$ will only provide additional information about $\PP (Y | X) $ if we are working in an anti-causal prediction problem $Y \rightarrow X$. This leads the authors to hypothesise that additional unlabelled semi-supervised samples will be most helpful in the anti-causal direction. In our work, we go further and prove a concrete generalisation benefit from using additional samples from $\PP(X)$ when the data generating process follows a collider, a graphical structure which is inherently anti-causal as it relies on $Y$ having shared children with another vertex.



\paragraph{Independence Regularisation and Fair Learning} 
Our work is related to the large body of recent work aiming to force conditional independence constraints, either for fairness~\cite{kamishima2011fairness} or domain generalisation~\cite{pogodin2022efficient}. However, it is important to note that if $Y$ satisfies a conditional independence this does not mean that the optimal least-square regressor $\EE[Y|X]$ will satisfy the same conditional independence. For example, let
\begin{equation}
    \left\{
    \begin{aligned}
        \begin{split}
            & Y, X_2 \sim \cN(0,1) \text{ with } Y\indep X_2 \\
            & X_1 = Y \mathbbm{1} \{X_2>0\}.
        \end{split}
    \end{aligned}
    \right.
\end{equation}
Then we have $\EE[Y|X_1, X_2] = X_1\mathbbm{1}\{X_2>0\}$, hence $\EE[Y|X_1, X_2]$ is constant when $X_2<0$ but not otherwise. Therefore $\EE[Y|X_1, X_2] \not\indep X_2$, even though $Y\indep X_2$. 

Therefore, our methodology is more similar to ensuring independence in expectation. Specifically, the RKHS methodology is related to work on fair kernel learning~\cite{perez2017fair,li2022kernel}. However, in contrast to the work on fair kernel learning where regularisation terms for encouraging independence are proposed, we go further by enforcing the mean independence constraint directly onto the hypothesis space.

\paragraph{Availability of DAG as prior knowledge}

Our work is based on the premise of having exact knowledge of the DAG underlying the data generating structure. This knowledge typically comes from domain expertise, with examples in genetics~\cite{day2016robust} or in the aerosol radiative forcing experiment we present. However, when domain expertise is insufficient, we may need causal discovery methods to uncover the causal relationships. These methods can be expensive to run at large scale and can provide a DAG with missing or extra edges when compared to the true DAG. If collider regression is run with a partially incorrect DAG, it is likely that it would degrade the performance, as such a setting would amount to introducing incorrect prior information in the model. However, if the estimated DAG is \enquote{close} to the true DAG in the sense of the independence relationships they induce, then there may still be benefit in the finite sample regime.

\paragraph{Generality of proposed method}

Two aspects of the methodology introduced in Section~\ref{section:more-general-dag} need to be caveated. First, it is important to require there exists no edge from children of $Y$ to spouses of $Y$, otherwise that would break the conditional independence $Y \indep X_2|X_3$. Second, whilst this is a general procedure that provides a useful inductive bias and helps restrict the hypothesis class, this procedure may not account for all the possible inductive biases that arise from the DAG at its most granular level. The procedure accounts for the collider constraint that arises from grouping variables together, not for every collider structure that might exist in the DAG. Encoding more granular collider structure would require additional regression steps, and a systematic way to perform such additional steps remains an interesting avenue for further research.

\vfill
\section{Conclusion}

In this work we have demonstrated that collider structures within causal graphs constitute a useful form of inductive bias for regression that benefits generalisation performance. Whilst we focused on least-square regression, we expect that the collider regression framework should benefit a wider range of machine learning problems that aim to make inferences about $\PP(Y|X)$. For example, a natural extension of this work should investigate collider regression for classification or quantile regression tasks.

\section*{Acknowledgements}

The authors would like to thank Bryn Elesedy, Dimitri Meunier, Siu Lun Chau, Jean-François Ton, Christopher Williams, Duncan Watson-Parris, Eugenio Clerico\footnote{ and Tyler Farghly} and Arthur Gretton for many helpful discussions and valuable feedbacks. Shahine Bouabid receives funding from the European Union’s Horizon 2020 research and innovation programme under Marie Skłodowska-Curie grant agreement No 860100. Jake Fawkes receives funding from the EPSRC.

\bibliography{mybib}
\bibliographystyle{icml2021}

\appendix

\onecolumn

\newpage

\section{Notations and useful Results}\label{appendix:lemmas}

\subsection{Notations}

Let $\cX$ be a Borel space,  $\cY\subseteq\RR$ and let $X$ and $Y$ be random variables valued in $\cX$ and $\cY$. We denote $(L^2(X), \langle \cdot, \cdot\rangle_{L^2(X)})$ the Hilbert space of functions from $\cX$ to $\RR$ which are square-integrable with respect to the pushforward measure induced by $X$, i.e.\ $\PP_X = \PP\circ X^{-1}$.

Let $(\cH, \langle \cdot, \cdot\rangle_\cH)$ be a RKHS of functions from $\cX$ to $\RR$ with reproducing kernel $k : \cX\times\cX\to\RR$. We denote its canonical feature map as $k_x$ for any $x\in\cX$.

Let $\bA \in \RR^{n\times n}$, we denote $\lambda_{\min}(\bA)$ and $\lambda_{\max}(\bA)$ the smallest and largest eigenvalues of $\bA$ respectively.

\subsection{Useful results}

\begin{theorem}[Theorem 3.11, {\cite{paulsen2016introduction}}]\label{lemma:constant-function}
    Let $\cH$ be a RKHS on $\cX$ with reproducing kernel $k$ and let $f :\cX\to\RR$. Then the following are equivalent:
    \begin{enumerate}
        \item[(i)] $f\in\cH$
        \item[(ii)] there exists $c\geq 0$ such that $c^2 k(x, y) - f(x)f(y)$ is kernel function
    \end{enumerate}
\end{theorem}

\begin{lemma}[Corollary 5.5, {\cite{paulsen2016introduction}}]\label{lemma:orthogonal-constant}
    Let $\cH_1, \cH_2$ be RKHS on $\cX$ with reproducing kernels $k_1, k_2$. If $\cH_1\cap\cH_2 = \{0\}$, then $\cH = \cH_1 \oplus\cH_2$ is a RKHS with reproducing kernel $k = k_1 + k_2$ and $\cH_1, \cH_2$ are orthogonal subspaces of $\cH$.
\end{lemma}

\begin{proposition}\label{lemma:rkhs-feature-map}
    Let $(\cV, \langle \cdot, \cdot\rangle_\cV)$ be a Hilbert space, $\varphi : \cX\to\cV$ be a mapping function and
    \begin{equation}
        k(x, y) = \langle \varphi(x), \varphi(y)\rangle_\cV,\quad x, y\in\cX
    \end{equation}
    the kernel function induced by $\varphi$. Then the RKHS induced by $k$ is given by
    \begin{equation}
        \cH = \{x\mapsto \langle v, \varphi(x)\rangle_\cV \mid v\in\cV\}.
    \end{equation}
    \begin{proof}
        The proof follows from the application of the Pull-back Theorem~{[Theorem 5.7]\cite{paulsen2016introduction}} to the linear kernel $L : \cV\times\cV\to\RR, (v, v')\mapsto \langle v, v'\rangle_\cV$ composed with the feature map $\varphi : \cX\to\cV$.
    \end{proof}
\end{proposition}

\begin{lemma}\label{lemma:quadratic-conditional-expectation}
    Suppose $\bA \in \RR^{n\times n}$ is symmetric, let $Z$ be a random variable, $\bx\in\RR^n$ be a random vector. 
    \begin{equation}
        \EE[\bx^\top \bA \bx\mid Z]  = \Tr\left(\bA\operatorname{Var}(\bx|Z)\right) + \EE[\bx|Z]^\top \bA \EE[\bx|Z].
    \end{equation}
\end{lemma}
\begin{proof}
    \begin{align}
        \EE[\bx^\top \bA \bx\mid Z] & = \EE\left[\Tr\left(\bA \bx\bx^\top\right) \mid Z\right] \\
        & = \Tr\left(\bA \EE[\bx\bx^\top | Z]\right) \\
        & = \Tr\left(\bA \left(\operatorname{Var}(\bx|Z) + \EE[\bx|Z]\EE[\bx|Z]^\top\right)\right) \\
        & = \Tr\left(\bA\operatorname{Var}(\bx|Z)\right) + \Tr\left(\bA \EE[\bx|Z]\EE[\bx|Z]^\top\right) \\
        & = \Tr\left(\bA\operatorname{Var}(\bx|Z)\right) + \EE[\bx|Z]^\top \bA \EE[\bx|Z].
    \end{align}
\end{proof}

\begin{lemma}[{\cite{mori1988comments}}]\label{lemma:lambda-min}
    Let $\bA, \bB \in\RR^{n\times n}$ and suppose $\bA$ symmetric and $\bB$ positive semi-definite, then
    \begin{equation}
        \Tr(\bA\bB) \geq \lambda_{\min}(\bA)\Tr(\bB).
    \end{equation}
\end{lemma}

\begin{lemma}[Lemma B.3, {\cite{elesedy2021provably}}]\label{lemma:norm-max-matrix}
    Let $\bA\in\RR^{n\times n}$, then
    \begin{equation}
        \lambda_{\max}(\bA) \leq n \max_{i, j} |\bA_{ij}|.
    \end{equation}
\end{lemma}

\newpage
\section{Supporting proofs}

\subsection{Notations}

We start by introducing measure-theoretic notations which will be of use in the supporting proofs.

Let $(\Omega, \mfF, \PP)$ denote a probability space, we denote $L^2(\Omega, \mfF, \PP)$ the space of random variables with finite variance, which we will denote $L^2(\Omega)$ for conciseness when the $\sigma$-algebra is $\mfF$. Endowed with inner product $\langle Z, Z'\rangle_{L^2(\Omega)} = \EE[ZZ']$, $L^2(\Omega)$ has a Hilbert structure. For any random variable $Z$, we denote $\sigma(Z)\subset\mfF$ the $\sigma$-algebra generated by $Z$.

\subsection{Proofs of Proposition~\ref{proposition:L2-generalisation-benefit}}


\begin{customprop}{\ref{proposition:L2-generalisation-benefit}}
    Let $h \in L^2(X)$ be any regressor from our hypothesis space. We have
    \begin{equation}
        \Delta(h, Ph)  = \|Eh\|_{L^2(X)}^2.
    \end{equation}
\end{customprop}
\begin{proof}
    The conditional expectation $\Pi : Z\in L^2(\Omega) \mapsto \EE[Z|X_2]$ defines an orthogonal projection onto the space of $X_2$-measurable random variables with finite variance $L^2(\Omega, \sigma(X_2), \PP)$. Thus, its range and null space are orthogonal in $L^2(\Omega)$.

    Let $h \in L^2(X)$. We have $Eh(X) = \EE[h(X)|X_2] = \Pi h(X)$ hence $Eh(X)$ is in the range of $\Pi$. On the other hand,
    \begin{equation}
        \EE[Ph(X)|X_2] = \EE[h(X)|X_2] - \EE[Eh(X)|X_2] = \EE[h(X)|X_2] - \EE[h(X)|X_2] = 0,
    \end{equation}
    therefore $Ph(X)$ is in the null space of $\Pi$. Finally, because $Y\indep X_2$ we have $\EE[Y|X_2] = \EE[Y] = 0$ by assumption, therefore $Y$ is also in the null space of $\Pi$.

    Hence, adopting this random variable view, the desired result simply follows from $L^2(\Omega)$ orthogonality:
    \begin{align*}
        \Delta(h, Ph) & = \EE[(Y - h(X))^2] - \EE[(Y - Ph(X))^2] \\
        & = \|Y - h(X)\|_{L^2(\Omega)}^2 - \|Y - Ph(X)\|_{L^2(\Omega)}^2 \\
        & = \|Y - Ph(X) - Eh(X)\|_{L^2(\Omega)}^2 - \|Y - Ph(X)\|_{L^2(\Omega)}^2 \\
        & = \|Y - Ph(X)\|_{L^2(\Omega)}^2 + \|Eh(X)\|_{L^2(\Omega)}^2 - \|Y - Ph(X)\|_{L^2(\Omega)}^2 \\
        & = \EE[Eh(X)^2] \\
        & = \|Eh\|_{L^2(X)}^2.
    \end{align*}
\end{proof}

\subsection{Proofs of Proposition~\ref{proposition:projected-RKHS}}

\begin{customprop}{\ref{proposition:projected-RKHS}}
    Let $P^*$ be the adjoint operator of $P$ in $\cH$. Then $\cH_P$ is also a RKHS with reproducing kernel
    \begin{equation}
        k_P(x,x') = \langle P^* k_x, P^* k_{x'}\rangle_\cH
    \end{equation}
     with $P^* k_x = k_x - \mu_{X|X_2=\pi_2(x)}$.
\end{customprop}
\begin{proof}[Proof of Proposition~\ref{proposition:projected-RKHS}]
    Let $\cH_P$ denote the reproducing kernel with $k_P$. We start by showing that $P\cH \subseteq \cH_P$.

    Let $f \in P\cH$, then it admits a pre-image $w_f\in\cH$ such that $f = P w_f$. Hence for any $x\in\cX$, we get that
    \begin{equation}
        f(x) = \langle f, k_x \rangle_\cH = \langle Pw_f, k_x\rangle_\cH = \langle w_f, P^* k_x\rangle_\cH.
    \end{equation}
    Hence, $f$ writes as an element of the RKHS induced by the feature map $x\mapsto P^*k_x$ and by Proposition~\ref{lemma:rkhs-feature-map} $f\in\cH_P$.

    Reciprocally, let us now show that $\cH_P\subseteq P\cH$. Let $f\in\cH_P$, again by Proposition~\ref{lemma:rkhs-feature-map} there exists $w_f\in\cH$ such that for any $x\in\cX$,
    \begin{equation}
        f(x) = \langle w_f, P^*k_x\rangle_\cH = Pw_f(x).
    \end{equation}
    This proves that $f\in P\cH$ which concludes the proof.
\end{proof}

\subsection{Proofs of Theorem~\ref{thm:main}}




\begin{customthm}{\ref{thm:main}}
    Suppose $M = \sup_{x\in\cX} k(x, x) < \infty$ and $\operatorname{Var}(Y|X) \geq \eta > 0$. Then, the generalisation gap between $\hat f$ and $P\hat f$ satisfies
    \begin{equation}
         \EE[\Delta(\hat f, P\hat f)] \geq \frac{\eta \EE\big[\|\mu_{X|X_2}(X)\|_{L^2(X)}^2\big]}{\left(\sqrt{n}M + \lambda / \sqrt{n}\right)^2}
    \end{equation}
    where $\mu_{X|X_2} = \EE[k_X|X_2]$ is the CME of $\PP(X|X_2)$.
\end{customthm}
\begin{proof}[Proof of Theorem~\ref{thm:main}]

    Let $\Pi = \EE[\cdot|X_2]$ be the $L^2(\Omega)$ orthogonal projection onto the subspace of $X_2$-measurable random variables. For any $h\in L^2(X)$, we verify that $Eh(X) = \EE[h(X)|X_2] = \Pi[h(X)]$ hence $Eh(X)\in \Range(\Pi)$. Furthermore, because $Y \indep X_2$ we have $\Pi[Y] = \EE[Y|X_2] = \EE[Y] = 0$ by assumption, hence $Y \in \Ker(\Pi)$.

    Now let
    \begin{equation}
        \bx = \begin{bmatrix} X^{(1)} \\ \vdots  \\ X^{(n)}\end{bmatrix}, \qquad \by = \begin{bmatrix} Y^{(1)} \\ \hdots \\ Y^{(n)}\end{bmatrix}
    \end{equation}
    denote vectors of $n$ independent copies of $X$ and $Y$ and let
    \begin{equation}
        j(x, x') = \langle E k_x,  E k_{x'} \rangle_{L^2(X)} = \EE[Ek_x(X)Ek_{x'}(X)]\enspace \forall x, x'\in\cX.
    \end{equation}
    be the positive definite kernel induced by $L^2(X)$ inner product of $Ek_x$ and 
    \begin{equation}
        \bJ = j(\bx, \bx) = \left[j(X^{(i)}, X^{(j)})\right]_{1\leq i, j\leq n}
    \end{equation}
    the resulting Gram-matrix.

    Using notations from Section~\ref{subsection:rkhs-case}, we know the solution of the kernel ridge regression problem in $\cH$ takes the form
    \begin{equation}
        \hat f = \by^\top \left(\bK + \lambda \bI_n\right)^{-1}\bk_{\bx}.
    \end{equation}

    Hence, by linearity of the projection, we have
    \begin{equation}
        E\hat f = \by^\top \left(\bK + \lambda \bI_n\right)^{-1}E\bk_{\bx}
    \end{equation}
    with notation abuse $E\bk_{\bx} = [Ek_{X^{(1)}} \hdots Ek_{X^{(n)}}]^\top$.

    Therefore, we can write
    \begin{align}
        \Delta(\hat f, \hat Pf) & = \|E\hat f\|^2_{L^2(X)} \\
                                & = \EE_X[E\hat f(X)^2] \\
                                & = \EE_X\left[\left(\by^\top \left(\bK + \lambda \bI_n\right)^{-1}E\bk_{\bx}(X)\right)^2\right] \\
                                & = \EE_X\big[\by^\top \left(\bK + \lambda \bI_n\right)^{-1}E\bk_{\bx}(X)E\bk_{\bx}(X)^\top(\bK + \lambda \bI_n)^{-1}\by\big] \\
                                & = \by^\top \left(\bK + \lambda \bI_n\right)^{-1} \EE_X[E\bk_{\bx}(X)E\bk_{\bx}(X)^\top](\bK + \lambda \bI_n)^{-1}\by \\
                                & = \by^\top \left(\bK + \lambda \bI_n\right)^{-1}\bJ (\bK + \lambda \bI_n)^{-1}\by.
    \end{align}

    Let us now denote for conciseness $\bA = \left(\bK + \lambda \bI_n\right)^{-1}\bJ (\bK + \lambda \bI_n)^{-1}$. We have by Lemma XX,
    \begin{align}
        \EE_{\by}[\Delta(\hat f, P\hat f) \mid \bx] & = \EE_{\by}[\by^\top \bA\by \mid \bx] \\
        & = \Tr(\bA \operatorname{Var}(\by|\bx)) + \EE[\by|\bx]^\top \bA \EE[\by|\bx] & \text{Lemma~\ref{lemma:quadratic-conditional-expectation}}\\
        & \geq \Tr(\bA \operatorname{Var}(\by|\bx)),
    \end{align}
    where the conditional variance is the diagonal matrix given by
    \begin{align}
        \operatorname{Var}(\by|\bx)  = \begin{bmatrix}
            \operatorname{Var}(Y^{(1)}|X^{(1)}) & & \\
            & \ddots & \\
            & & \operatorname{Var}(Y^{(n)}|X^{(n)})
          \end{bmatrix}
    \end{align}
    because the copies of $(X, Y)$ are mutually independent.

    We therefore obtain,
    \begin{align}
        \EE_{\by}[\Delta(\hat f, P\hat f) \mid \bx] & \geq \Tr(\bA \operatorname{Var}(\by|\bx)) \\
        & \geq  \min_i\operatorname{Var}(Y^{(i)}|X^{(i)}) \Tr(\bA) \\
        & \geq \eta \Tr(\bA) \\
        & = \eta\Tr\big(\left(\bK + \lambda \bI_n\right)^{-1}\bJ (\bK + \lambda \bI_n)^{-1}\big) \\
        & \geq \eta \lambda_{\min}((\bK + \lambda\bI_n)^{-1})^2\Tr(\bJ)  & \text{Lemma~\ref{lemma:lambda-min}} \\
        & \geq \eta\frac{\Tr(\bJ)}{(Mn + \lambda)^2} & \text{Lemma~\ref{lemma:norm-max-matrix}}.
    \end{align}

    Finally taking the expectation against $\bx$, we get
    \begin{align}
        \EE[\Delta(\hat f, P\hat f)] & \geq \frac{\EE_{\bx}[ \eta \Tr(\bJ)]}{(Mn + \lambda)^2} \\
        &  = \frac{\eta \sum_{i=1}^n \EE_{X^{(i)}}[ j(X^{(i)}, X^{(i)})] }{(Mn + \lambda)^2} \\
        & = \frac{n\eta \EE[j(X, X)]}{(Mn + \lambda)^2} \\
        & = \frac{\eta \EE[j(X, X)]}{(M\sqrt{n} + \lambda/\sqrt{n})^2} \\
        & = \frac{\eta \EE\left[\|Ek_X\|_{L^2(X)}^2\right]}{(M\sqrt{n} + \lambda/\sqrt{n})^2} \\
    \end{align}

    Now, for our particular choice of projection $E$, we have for any $x\in\cX$ that
    \begin{align}
        Ek_X(x) & = \EE_{X'}[k_X(X')|X_2=\pi_2(x)] \\
                & = \EE_{X'}[k(X, X')|X_2=\pi_2(x)] \\
                & = \langle k_X, \EE_{X'}[k_{X'}|X_2=\pi_2(x)]\rangle_\cH \\
                & = \langle k_X, \mu_{X|X_2=\pi_2(x)}\rangle_\cH \\
                & = \mu_{X|X_2=\pi_2(x)}(X)
    \end{align}
    Therefore using the measure-theoretical CME notation from \cite{park2020measure}, we have
    \begin{equation}
        \|Ek_X\|_{L^2(X)}^2 = \EE_{X'}[Ek_X(X')^2] = \EE_{X'}[\mu_{X|X_2=\pi_2(X')}(X)^2] = \|\mu_{X|X_2}(X)\|_{L^2(X)}^2
    \end{equation}
    which concludes the proof.
\end{proof}

\subsection{Proof of Proposition~\ref{proposition:general-L2-generalisation-benefit}}
\begin{customprop}{\ref{proposition:general-L2-generalisation-benefit}}
    Let $h \in L^2(X)$ be any regressor from our hypothesis space. We have
    \begin{equation}
        \Delta(h, f_0 + P'h)  = \|E'h - f_0\|_{L^2(X)}^2.
    \end{equation}
\end{customprop}
\begin{proof}
    This proof follows the same structure than the proof of Proposition~\ref{proposition:L2-generalisation-benefit}.

    Let $\Pi  = \EE[\cdot|X_2, X_3]$ be the $L^2(\Omega)$ orthogonal projection onto the subspace of $(X_2, X_3)$-measurable random variables with finite variance $L^2(\Omega, \sigma(X_2, X_3), \PP)$. We have that
    \begin{align}
        \Pi[Y - f_0(X)] & = \EE[Y|X_2, X_3] - \EE[f_0(X)|X_2, X_3] \\
        & = \EE[Y|X_2, X_3] - \EE[\EE[Y|X_3]|X_2, X_3] \\
        & = \EE[Y|X_2, X_3] - \EE[\EE[Y|X_2, X_3]|X_2, X_3] & (Y\indep X_2|X_3)\\
        & = 0,
    \end{align}
    therefore $Y - f_0(X) \in \Ker(\Pi)$. On the other hand, we can easily verify that for any $h\in L^2(X)$, we have $E'h(X)\in \Range(\Pi)$ and $P'h(X)\in \Ker(\Pi)$.

    Therefore, it follows by $L^2(\Omega)$ orthogonality that for any $h \in L^2(X)$
    \begin{align*}
        \|Y - h(X)\|_{L^2(\Omega)}^2 & = \|(Y - f_0(X)) - (h(X) - f_0(X))\|^2_{L^2(\Omega)} \\
        & = \|(Y - f_0(X)) - P'(h - f_0)(X)\|^2_{L^2(\Omega)} + \|E'(h - f_0)(X)\|^2_{L^2(\Omega)} \\
        & = \|Y - (f_0(X) + P'h(X))\|^2_{L^2(\Omega)} + \|E'h - f_0\|^2_{L^2(X)} & (f_0\in\Range(E') = \Ker(P')).
    \end{align*}

    Which allows to conclude that
    \begin{align*}
        \Delta(h, f_0 + P'h) & = \EE[(Y - h(X))^2] - \EE[\big(Y - (f_0(X) + P'h(X))\big)^2] \\
        & = \|Y - h(X)\|_{L^2(\Omega)}^2 - \|Y - (f_0(X) + P'h(X))\|^2_{L^2(\Omega)} \\
        & = \|E'h - f_0\|^2_{L^2(X)}.
    \end{align*}
\end{proof}

\newpage
\section{Conditions for $P:\cH\to\cH$ to be well-defined}\label{appendix:rkhs-assumptions}

Let $\cH$ be a RKHS of real-valued functions over $\cX=\cX_1\times \cX_2$ with reproducing kernel $k : \cX\times\cX\to\RR$. In this section, we discuss conditions under which the orthogonal projection $P : L^2(X)\to L^2(X)$ can be seen as a well-defined projection over $\cH\subset L^2(X)$.

Formally, let $\iota : \cH \to L^2(X)$ denote the inclusion operator that maps elements of the RKHS $\cH\ni f\mapsto [f]_{\sim}$ to their equivalence class in $L^2(X)$. Saying that $P$ is well-defined as a projection over $\cH$ means that
\begin{equation}
    P\iota f \in \iota \cH\enspace\forall f\in\cH.
\end{equation}

Such construction however raises two issues
\begin{enumerate}
    \item Since $P = \Id - E$ and $E\iota f : x\mapsto \EE[\iota f(X)|X_2=x_2]$ is a function of $x_2$ only, for $E\iota f$ to lie in RKHS it is necessary for $\cH$ to contain functions that are constant with respect to $x_1$.
    \item If $f\in\cH$, there is no guarantee that $E\iota f = \EE[\iota f(X)|X_2=\pi(\cdot)]$ will also lie in $\cH$. In fact, this will often not be true --- e.g. when $\cX$ is a continuous domain~\cite{song2009hilbert} --- and we only have $P\iota\cH\subset L^2(X)$.
\end{enumerate}

In what follows, we permit ourselves to drop the $\iota$ notation.

\subsection{Issue 1 : $\cH$ must contain functions constant wrt $x_1$}

In general, it is not guaranteed that a RKHS will contain constant functions. In fact, this is not the case for generic RKHSs such as the RKHSs induced by Gaussian or Matérn kernels~\cite{steinwart2008support}. To overcome this issue, we propose a particular form for the reproducing kernel that will ensure the RKHS contains constant functions with respect to $x_1$.

\begin{proposition}
   Let $r :\cX_1\times\cX_1\to\RR$ and $\ell:\cX_2\times\cX_2\to\RR$ be kernel functions. Then the RKHS with reproducing kernel
   \begin{equation}
       k = (r + 1)\otimes\ell
   \end{equation}
   contains functions that are constant with respect to the first variable $x_1$.
\end{proposition}
\begin{proof}
    Let $r :\cX_1\times\cX_1\to\RR$ be a kernel function on $\cX_1$ and consider the kernel defined by $r^+ = r + 1$ with RKHS $\cH_{r^+}$. Let $c\in\RR$ and consider the constant function $g(x_1) = c\enspace\forall x_1\in\cX_1$.

    Then for any $x_1, x_1'\in\cX_1$ we have
    \begin{align}
        c^2 r^+(x_1, x_1') - g(x_1)g(x_1') & = c^2 r(x_1, x_1') + c^2 - c^2 \\
        & = c^2 r(x_1, x_1')
    \end{align}
    which is a kernel function. By Theorem~\ref{lemma:constant-function} we conclude that $\cH_{r^+}$ contains constant functions.

    We now consider a second kernel $\ell:\cX_2\times\cX_2\to\RR$ with RKHS $\cH_\ell$ and we propose to take $\cH$ as the tensor product RKHS
    \begin{equation}
        \cH = \cH_{r^+}\otimes\cH_\ell,
    \end{equation}
    which will have reproducing kernel
    \begin{equation}
        k = r^+\otimes \ell.
    \end{equation}
    Functions from $\cH$ now contain functions which are the product of functions from $\cH_{r^+}$ and $\cH_\ell$ and are therefore allowed to be constant with respect to $x_1$ since $\cH_{r^+}$ contains constant functions. 
\end{proof}

Note that while this structural assumption may appear to limit the generality of the proposed methodology, tensor product RKHSs are a widely used form of RKHS~\cite{szabo2017characteristic, pogodin2022efficient, lun2021rkhs} that preserve universality of kernels from individual dimension and provide a rich function space.

Recall now the expression of the finite sample $P^*$ estimate used in (\ref{eq:RKHS-pf-estimate}),
\begin{equation}
    \hat P^* = \Id - \bk_\bx^\top (\bL + \gamma\bI_n)^{-1}\bell_{\bx_2}.
\end{equation}

This allows to estimate the projected kernel $k_P$ following
\begin{align*}
    \hat k_P(x, x') & = \langle \hat P^*k_x, \hat P^*k_x\rangle_\cH \\
                    & = \left\langle k_x - \bk_\bx^\top (\bL + \gamma\bI_n)^{-1}\bell_{\bx_2}(x_2)  , k_{x'} - \bk_\bx^\top (\bL + \gamma\bI_n)^{-1}\bell_{\bx_2}(x_2') \right\rangle_\cH \\
                    & = k(x, x') \\
                    & - \bell_{\bx_2}(x_2)^\top(\bL + \gamma\bI_n)^{-1}\bk_x(x') \\
                    & - \bell_{\bx_2}(x_2')^\top(\bL + \gamma\bI_n)^{-1}\bk_x(x) \\
                    & - \bell_{\bx_2}(x_2)^\top(\bL + \gamma\bI_n)^{-1}\bK(\bL + \gamma\bI_n)^{-1}\bell_{\bx_2}(x_2').
\end{align*}

However, for the above derivation to be correct, we need that evaluations of the second kernel $\ell$ can be obtained by taking an inner product in $\cH$. Namely, we need that
\begin{equation}
    \ell(x_2, x_2') = \langle \ell_{x_2}, \ell_{x_2'}\rangle_{\cH_\ell} = \langle \ell_{x_2}, \ell_{x_2'}\rangle_{\cH}.
\end{equation}

The following result shows that a sufficient condition for this to hold is that $\cH_r$ itself does not contain constant functions. As mentioned above, this is a condition satisfied by generic RKHSs such as the RKHSs of the Gaussian kernel or the Matérn kernels~\cite{steinwart2008support} --- which is the RKHS we work with in our experiments.

\begin{proposition}
    Let $\cH = \cH_{r^+}\otimes\cH_\ell$ where $r^+ = r + 1$. If $\cH_r$ does not contain constant functions, then we have that $\ell(x_2, x_2') = \langle \ell_{x_2}, \ell_{x_2'}\rangle_{\cH}$.
\end{proposition}
\begin{proof}
    The kernel $r^{+}(x_{1},x_{1}')=r(x_{1},x_{1}')+1$ here induces a RKHS $\mathcal{H}_{r^{+}}$ (of functions from $\mathcal{X}_{1}$
to $\mathbb{R}$) which does contain constant functions, e.g., $e\in\mathcal{H}_{r^{+}}$,
where $e(x_{1})=1,$ $\forall x_{1}\in\mathcal{X}_{1}.$

This choice of kernel ensures that $\ell_{x_2}\in\cH$ when viewed as a function on $\cX_1\times\cX_2$, i.e.\ we can write it as $e\otimes\ell_{x_2}$, so it is clear that it belongs to $\mathcal{H} = \mathcal{H}_{r^{+}}\otimes\mathcal{H}_{\ell}$, since $e\in\mathcal{H}_{r^{+}}$ and $\ell_{x_2}\in\mathcal{H_{\ell}}$.

Furthermore, we have
\begin{align*}
    \langle \ell_{x_2}, \ell_{x_2'}\rangle_\cH & = \langle e\otimes \ell_{x_2}, e\otimes\ell_{x_2'}\rangle_\cH \\
    & = \langle e, e\rangle_{\cH_{r^+}} \langle \ell_{x_2}, \ell_{x_2'}\rangle_{\cH_\ell} \\
    & = \langle e, e\rangle_{\cH_{r^+}} \ell(x_2, x_2').
\end{align*}

However,
\begin{align*}
    \langle e, e\rangle_{\cH_{r^+}} & = \langle e, e + r_{x_1}\rangle_{\cH_{r^+}} - \langle e, r_{x_1}\rangle_{\cH_{r^+}} \\ 
    & = \langle e, r^+_{x_1}\rangle_{\cH_{r^+}} - \langle e, r_{x_1}\rangle_{\cH_{r^+}} \\
    & = e(x_1) - \langle e, r_{x_1}\rangle_{\cH_{r^+}} \\
    & = 1 - \langle e, r_{x_1}\rangle_{\cH_{r^+}}.
\end{align*}

Now if $\cH_r$ does not contain constant functions, we have $\operatorname{Span}(\{e\})\cap \cH_r = \{0\}$. Hence, by Lemma~\ref{lemma:orthogonal-constant} we obtain that $e$ and $r_{x_1}$ are orthogonal in $\cH_{r^+}$ which in turn gives that
\begin{equation}
    \langle e, r_{x_1}\rangle_{\cH_{r^+}} = 0 \Rightarrow \langle e, e\rangle_{\cH_{r^+}} = 1.
\end{equation}

Therefore, if $\cH_r$ does not contain constant functions we have that
\begin{equation}
     \langle \ell_{x_2}, \ell_{x_2'}\rangle_{\cH} = \langle e, e\rangle_{\cH_{r^+}}\ell(x_2, x_2') = \ell(x_2, x_2').
\end{equation}
\end{proof}

\subsection{Issue 2 : $P$ is not necessarily closed as an operator on $\cH$}

\paragraph{Too Long; Didn't Read} We make the assumption that $\EE[f(X)|X_2=\cdot] \in \cH$ for $f\in\cH$.

\paragraph{Too Short; Want More}

It is possible to choose the reproducing kernel $k$ such that $\cH$ is dense in $L^2(X)$. This property is called $L^2$-universality~\cite{sriperumbudur2011universality}. Whilst this might suggest that the assumption $Ef\in\cH$ for $f\in\cH$ could be reasonable when $L^2$-universality is met, in practice no explicit case is provided in the literature where it is easy to verify that $\EE[f(X)|X_2=\cdot]\in\cH$ for $f\in\cH$.

In fact, a classic counter example given by \citet{fukumizu2013kernel} is the case where $\cH$ is the RKHS of the Gaussian kernel on $\cX$ and $X\indep Z$. Then, $\EE[f(X)|Z=\cdot]$ is constant for any $f\in\cH$ but $\cH$ does not contain constant functions~\cite{steinwart2008support}. In the context of our work, we do not have $(X_1, X_2)\indep X_2$ but it nonetheless remains difficult to verify whether $\EE[f(X)|X_2=\cdot]\in\cH$.

Efforts to study this nontrivial research direction must be highlighted : \citet{mollenhauer2020nonparametric} show that under denseness assumptions, it is possible to approximate the conditional expectation operator $E : L^2(X)\to L^2(X)$ with a Hilbert-Schmidt operator on $\cH$ with arbitrary precision. \citet{klebanov2020rigorous} propose a rigorous RKHS-friendly construction of $E$ that only assumes that $Ef$ lies a constant away from $\cH$. Most recently, \citet{li2022optimal} consider the weaker assumption that for $f\in\cH$, $Ef$ lies in an interpolation space between $\cH$ and $L^2(X)$ and prove optimal learning rates for its estimator.

The theoretical intricacies of such considerations tend however to undermine more \enquote{practical}-driven work. For this reason, it is common to defer such consideration to theoretical research and make the assumption that $\EE[f(X)|X_2=\cdot]\in\cH$~\cite{fukumizu2004dimensionality, song2011kernel, muandet2016kernel, hsu2019bayesian, ton2021noise, chau2021deconditional, fawkes2022doubly}. Since the RKHS theory is not central to our motivations but only a tool we use to demonstrate the benefits of collider regression, we propose to make a similar assumption and delegate this theoretical consideration for future work.

\newpage
\section{Collider Regression on a simple DAG: estimators}\label{appendix:estimators-details}

Let $k : \cX\times\cX\to\RR$ and $\ell:\cX_2\times\cX_2\to\RR$ be positive definite kernel. In what follows, we adopt notations from the Section~\ref{subsection:rkhs-case}. $\hat f = \by^\top (\bK + \lambda \bI_n)^{-1}\bk_\bx$ denotes the solution to the kernel ridge regression problem in $\cH$. We abuse notation and denote the pairwise inner product of feature maps as
\begin{equation}
    \langle \bk_\bx, \bk_\bx\rangle_\cH = \begin{bmatrix} \langle k_{x_i}, k_{x_j}\rangle_\cH\end{bmatrix}_{1\leq i, j\leq n} = \begin{bmatrix} k(x_i, x_j)\end{bmatrix}_{1\leq i, j\leq n} = \bK.
\end{equation}

\subsection{For a general choice of kernel $k$}

\subsubsection{Estimating $\mu_{X|X_2=x_2}$}

We are interested in estimating the CME $\mu_{X|X_2=x_2}$. Using the CME estimate from (\ref{eq:cmo-estimate}), we obtain
\begin{equation}\label{eq:boxed-cme-estimate}
    \boxed{\hat \mu_{X|X_2=x_2} = \bk_\bx^\top (\bL + \gamma \bI_n)^{-1} \bell_{\bx_2}(x_2).}
\end{equation}

\subsubsection{Estimating $P\hat f$}

Writing out
\begin{align}
    P\hat f(x_1, x_2) & = \hat f(x_1, x_2) - \langle \hat f, \mu_{X|X_2=x_2}\rangle_\cH \\
    & = \by^\top (\bK + \lambda \bI_n)^{-1}\bk_\bx(x_1, x_2) - \by^\top (\bK + \lambda\bI_n)^{-1}\langle \bk_\bx,  \mu_{X|X_2=x_2}\rangle_\cH,
\end{align}
it appears we can obtain an estimate of $P\hat f$ by substituting $\mu_{X|X_2=x_2}$ with its estimate in the above. We obtain
\begin{align}
    \hat P\hat f(x_1, x_2) & = \by^\top (\bK + \lambda \bI_n)^{-1}\bk_\bx(x_1, x_2) - \by^\top (\bK + \lambda\bI_n)^{-1}\underbrace{\langle \bk_\bx, \bk_\bx\rangle_\cH}_{\bK} (\bL + \gamma \bI_n)^{-1} \bell_{\bx_2}(x_2) \\
    & = \by^\top (\bK + \lambda \bI_n)^{-1} \left(\bk_\bx(x_1, x_2) - \bK (\bL + \gamma \bI_n)^{-1}\bell_{\bx_2}(x_2)\right),
\end{align}
or in functional form
\begin{equation}
    \boxed{\hat P \hat f \!=\!  \by^\top \!\left(\bK \!+\! \lambda \bI_n\right)^{-1}\!\left(\bk_\bx\! - \!\bK(\bL \!+\! \gamma \bI_n)^{-1}\!\bell_{\bx_2}\right).}
\end{equation}

\subsection{When $k = (r + 1)\otimes \ell$}

In Section~\ref{subsection:new-rkhs-case}, a sufficient assumption for the projection to be well-defined is that the kernel takes the form
\begin{equation}
    k = (r + 1)\otimes \ell,
\end{equation}
where $r : \cX_1\times\cX_1\to\RR$ is a positive definite kernel. When we choose this particular form of kernel, alternative estimators can be devised.

In what follow, we denote $r^+ = r + 1, \brplus_{\!\!\!\!\bx_1} = r^+(\bx_1, \cdot)$ and $\bRplus = r^+(\bx_1, \bx_1)$.

\subsubsection{Estimating $\mu_{X|X_2=x_2}$}

Going back to the definition of CMEs, we can write
\begin{equation}
    \mu_{X|X_2=x_2} = \EE[k_X|X_2=x_2] = \EE[r^+_{X_1}\otimes \ell_{X_2}|X_2=x_2] = \EE[r^+_{X_1}|X_2=x_2]\otimes \ell_{x_2} = \mu_{X_1|X_2=x_2}\otimes \ell_{x_2}.
\end{equation}

Therefore, it is sufficient to obtain an estimate of $\mu_{X_1|X_2=x_2}$, which we can get as
\begin{equation}
    \hat \mu_{X_1|X_2=x_2}  = \brplusx1^\top (\bL + \gamma \bI_n)^{-1} \bell_{\bx_2}(x_2),
\end{equation}
and take as a CME estimator
\begin{equation}
    \boxed{\hat \mu_{X|X_2=x_2} = \left[\brplusx1^\top (\bL + \gamma \bI_n)^{-1} \bell_{\bx_2}(x_2)\right]\ell_{x_2}(\cdot)}
\end{equation}

\subsubsection{Estimating $P\hat f$}

Following the similar derivations than in the general case, we obtain
\begin{align}
    \hat P\hat f(x_1, x_2) & = \by^\top (\bK + \lambda \bI_n)^{-1}\bk_\bx(x_1, x_2) - \by^\top (\bK + \lambda\bI_n)^{-1}\langle \brplusx1, \brplusx1\rangle_{\cH_{r^+}}\langle \ell_{\bx_2}, \ell_{x_2}\rangle_{\cH_\ell} (\bL + \gamma \bI_n)^{-1} \bell_{\bx_2}(x_2) \\
    & = \by^\top (\bK + \lambda \bI_n)^{-1}\left[\bk_\bx(x_1, x_2) - \operatorname{Diag}(\ell_{\bx_2}(x_2)) \bRplus (\bL + \gamma \bI_n)^{-1}\bell_{\bx_2}(x_2)\right],
\end{align}
where $\operatorname{Diag}(\ell_{\bx_2}(x_2))$ is the diagonal matrix that has the vector $\ell_{\bx_2}(x_2) = \ell(\bx_2, x_2)$ as its diagonal. Written in functional form we obtain
\begin{equation}
    \boxed{\hat P\hat f = \by^\top (\bK + \lambda \bI_n)^{-1}\left[\bk_\bx - \operatorname{Diag}(\ell_{\bx_2}(\cdot)) \bRplus (\bL + \gamma \bI_n)^{-1}\bell_{\bx_2}\right]}
\end{equation}

\subsubsection{Estimating $k_P$}

Writing out,
\begin{align}
     k_P(x, x') & = \langle P^*k_x, P^*k_{x'}\rangle_\cH \\
     & = \langle k_x - \mu_{X|X_2=x_2}, k_{x'} - \mu_{X|X_2=x_2'}\rangle_\cH \\
     & = \langle k_x, k_{x'}\rangle_\cH \\
     & - \langle \mu_{X_1|X_2=x_2}\otimes\ell_{x_2}, k_{x'}\rangle_\cH \\
     & - \langle k_x, \mu_{X_1|X_2-x_2'}\otimes\ell_{x_2'}\rangle_\cH \\
     & + \langle \mu_{X_1|X_2=x_2}\otimes\ell_{x_2}, \mu_{X_1|X_2-x_2'}\otimes\ell_{x_2'}\rangle_\cH \\
     & = r^+(x_1, x_1')\ell(x_2, x_2') \\
     & - \langle \mu_{X_1|X_2=x_2}, r^+_{x_1'}\rangle_{\cH_{r^+}}\ell(x_2, x_2') \\
     & -  \langle r^+_{x_1}, \mu_{X_1|X_2=x_2'}\rangle_{\cH_{r^+}}\ell(x_2, x_2') \\
     & + \langle \mu_{X_1|X_2=x_2}, \mu_{X_1|X_2=x_2'}\rangle_{\cH_{r^+}}\ell(x_2, x_2') \\
     & = \ell(x_2, x_2')\left[r^+(x_1, x_1') -  \langle \mu_{X_1|X_2=x_2}, r^+_{x_1'}\rangle_{\cH_{r^+}} - \langle r^+_{x_1}, \mu_{X_1|X_2=x_2'}\rangle_{\cH_{r^+}} + \langle \mu_{X_1|X_2=x_2}, \mu_{X_1|X_2=x_2'}\rangle_{\cH_{r^+}}\right]
\end{align}

Therefore, substituting $\mu_{X_1|X_2=x_2}$ with its estimate, we obtain
\begin{subequations}
\begin{empheq}[box=\widefbox]{align}
    \hat k_P(x, x') & = \ell(x_2, x_2') \\
                    & \times\big[ r^+(x_1, x_1') \\
                    & - \bell_{\bx_2}(x_2)^\top(\bL + \gamma\bI_n)^{-1}\brplusx1(x_1') \\
                    & - \bell_{\bx_2}(x_2')^\top(\bL + \gamma\bI_n)^{-1}\brplusx1(x_1) \\
                    & - \bell_{\bx_2}(x_2)^\top(\bL + \gamma\bI_n)^{-1}\bRplus(\bL + \gamma\bI_n)^{-1}\bell_{\bx_2}(x_2') \big].
\end{empheq}
\end{subequations}

\newpage
\section{Collider Regression on a general DAG: algorithms and estimators}\label{appendix:general-case-details}

Let $k : \cX\times\cX\to\RR$, $r : \cX_1\times\cX_1 \to\RR$ and $\ell : (\cX_2\times\cX_3)\times(\cX_2\times\cX_3)\to\RR$ be psd kernels. We follow the same notation convention that in the case of a simple collider, except that now $\ell$ is a kernel over $\cX_2\times\cX_3$. Define $f_{0}\left(x\right)=\mathbb{E}\left[Y|X_{3}=x_{3}\right]$. Here $g^*=f^*-f_{0}$ must live in the appropriate subspace of functions
which have zero conditional expectation on $\left(X_{2},X_{3}\right)$.

\subsection{Algorithms}

\begin{algorithm}[h]
    \caption{General procedure to estimate $f_0 + P'\hat g$}
    \label{alg:general-L2-collider-regression}
\begin{algorithmic}[1]
    \STATE Regress $X_3 \rightarrow Y$ to get  $x_3 \mapsto \hat f_0(x_3)$
\STATE Take $\tilde Y = Y - \hat f_0(X_3)$
    \STATE Regress $(X_1, X_2, X_3) \rightarrow \tilde Y$ to get $(x_1, x_2, x_3)\mapsto \hat g(x_1, x_2, x_3)$
    \STATE Regress $(X_2, X_3)\to \hat g(X_1, X_2, X_3)$ to get $(x_2, x_3)\mapsto \hat \EE[\hat g(X_1, X_2, X_3)|X_2=x_2, X_3=x_3]$
    \STATE Take $\hat P'\hat g(x_1, x_2, x_3) = \hat g(X_1, X_2, X_3) - \hat \EE[\hat g(x_1, x_2, x_3)|X_2=x_2, X_3=x_3]$
    \STATE \textbf{return} $\hat f_0 + \hat P'\hat g$
\end{algorithmic}
\end{algorithm}

\begin{algorithm}[h]
    \caption{RKHS procedure to estimate $f_0 + P' \hat g$}
    \label{alg:general-rkhs-collider-regression}
\begin{algorithmic}[1]
    \STATE Estimate $\hat \mu_{X|X_2=x_2, X_3=x_3}$
    \STATE Regress $X_3 \rightarrow Y$ to get  $x_3 \mapsto \hat f_0(x_3)$
    \STATE Take $\tilde \by = \by - \hat f_0(\bx_3)$
    \STATE Take $\hat g = \tilde\by^\top (\bK + \lambda\bI_n)^{-1}\bk_\bx$
    \STATE Let $\hat P' \hat g = \hat g  - \langle \hat g, \hat\mu_{X|X_2=\cdot, X_3=\cdot}\rangle_\cH$
    \STATE \textbf{return} $\hat f_0 +\hat P' \hat g$
\end{algorithmic}
\end{algorithm}

\begin{algorithm}[h]
    \caption{RKHS procedure to estimate $f_0 + \hat g_{P'}$}
    \label{alg:general-rkhs-collider-regression}
\begin{algorithmic}[1]
    \STATE Estimate $\hat \mu_{X|X_2=x_2, X_3=x_3}$
    \STATE Regress $X_3 \rightarrow Y$ to get  $x_3 \mapsto \hat f_0(x_3)$
\STATE Take $\tilde \by = \by - \hat f_0(\bx_3)$
    \STATE Let $\hat P'^* k_x = k_x - \hat\mu_{X|X_2=x_2, X_3=x_3}$
    \STATE Let $\hat k_{P'}(x, x') = \langle \hat P'^* k_x, \hat P'^* k_{x'}\rangle_\cH$
    \STATE Evaluate $\hat\bK_{P'} = \hat k_{P'}(\bx, \bx)$ and $\hat \bk_{P', \bx} = \hat k_{P'}(\bx, \cdot)$
    \STATE Take $\hat g_{P'} = \tilde\by^\top (\hat \bK_{P'} + \lambda\bI_n)^{-1}\hat \bk_{P', \bx}$
    \STATE \textbf{return} $\hat f_0 +\hat g_{P'}$
\end{algorithmic}
\end{algorithm}

\subsection{Estimators for a general kernel $k$}

\subsubsection{Estimating $\mu_{X|X_2=x_2, X_3=x_3}$}

\begin{equation}\label{eq:boxed-cme-estimate}
    \boxed{\hat \mu_{X|X_2=x_2, X_3=x_3} = \bk_\bx^\top (\bL + \gamma \bI_n)^{-1} \bell_{\bx_2, \bx_3}(x_2, x_3).}
\end{equation}

\subsubsection{Estimating $P'\hat g$}

\begin{equation}
    \boxed{\hat P' \hat g =  \tilde \by^\top \!\left(\bK \!+\! \lambda \bI_n\right)^{-1}\!\left(\bk_\bx\! - \!\bK(\bL \!+\! \gamma \bI_n)^{-1}\!\bell_{\bx_2, \bx_3}\right).}
\end{equation}

\subsection{Estimators when $k = (r+1)\otimes \ell$}

\subsubsection{Estimating $\mu_{X|X_2=x_2, X_3=x_3}$}
\begin{equation}
    \boxed{\hat \mu_{X|X_2=x_2} = \left[\brplusx1^\top (\bL + \gamma \bI_n)^{-1} \bell_{\bx_2, \bx_3}((x_2, x_3))\right]\ell_{x_2, x_3}(\cdot)}
\end{equation}

\subsubsection{Estimating $P'\hat g$}
\begin{equation}
    \boxed{\hat P'\hat g = \tilde \by^\top (\bK + \lambda \bI_n)^{-1}\left[\bk_\bx - \operatorname{Diag}(\ell_{\bx_2, \bx_3}(\cdot)) \bRplus (\bL + \gamma \bI_n)^{-1}\bell_{\bx_2, \bx_3}\right]}
\end{equation}

\subsubsection{Estimating $k_{P'}$}
\begin{subequations}
\begin{empheq}[box=\widefbox]{align}
    \hat k_{P'}(x, x') & = \ell\left((x_2, x_3), (x_2', x_3')\right) \\
                    & \times\big[ r^+(x_1, x_1') \\
                    & - \bell_{\bx_2, \bx_3}((x_2, x_3))^\top(\bL + \gamma\bI_n)^{-1}\brplusx1(x_1') \\
                    & - \bell_{\bx_2, \bx_3}((x_2', x_3'))^\top(\bL + \gamma\bI_n)^{-1}\brplusx1(x_1) \\
                    & - \bell_{\bx_2, \bx_3}((x_2, x_3))^\top(\bL + \gamma\bI_n)^{-1}\bRplus(\bL + \gamma\bI_n)^{-1}\bell_{\bx_2, \bx_3}((x_2', x_3')) \big].
\end{empheq}
\end{subequations}

\newpage
\section{Details on experiments}\label{appendix:experiments}

\subsection{Models}

\paragraph{\ding{227} RF} We use the scikit-learn~\cite{scikit-learn} \texttt{sklearn.ensemble.RandomForestRegressor} implementation which we tune for
\begin{itemize}
    \item \texttt{n\char`_ estimators}
    \item \texttt{max\char`_ depth}
    \item \texttt{min\char`_ samples\char`_ split}
    \item \texttt{min\char`_ samples\char`_ leaf}
\end{itemize}
using a cross-validated grid search over an independently generated validation set.

\paragraph{\ding{227} $P$-RF} Once RF has been fitted  as $\hat f$, we estimate $\EE[\hat f(X_1, X_2)|X_2]$ by fitting a linear regression model of $X_2$ onto $\hat f(X_1, X_2)$.

\paragraph{\ding{227} KRR} We implement our own kernel ridge regression in PyTorch~\cite{pytorch}. The kernel is taken as
\begin{equation}
    k\big((x_1, x_2), (x_1', x_2')\big) = \big(\kappa_{\theta_1}(x_1, x_1') + 1\big) \kappa_{\theta_2}(x_2, x_2'),
\end{equation}
where $\kappa_\theta$ denotes the Gaussian kernel with lengthscale $\theta > 0$
\begin{equation}
    \kappa_\theta(u, u') = \exp\left(-\frac{\|u - u'\|_2^2}{\theta}\right).
\end{equation}

The kernel lengthscales $\theta_1, \theta_2$ and the regularisation weight $\lambda > 0$ are tuned using a cross-validated grid search on an independently generated validation set.

\paragraph{\ding{227} $P$-KRR} Once KRR has been fitted as $\hat f = \by^\top \left(\bK + \lambda \bI_n\right)^{-1}\bk_\bx$, we estimate the CME and use it to estimate $P\hat f(x_1, x_2) = \hat f(x_1, x_2) - \langle \hat f, \mu_{X|X_2=x_2}\rangle_\cH$ following
\begin{align}
    & \hat\mu_{X|X_2=x_2} = \bk_\bx^\top (\bL + \gamma\bI_n)^{-1}\bell_{\bx_2}(x_2) \\
    \Rightarrow & \hat P = \Id - \hat\mu_{X|X_2=\cdot} \\
     & \quad = \Id - \bk_\bx^\top (\bL + \gamma\bI_n)^{-1}\bell_{\bx_2} \\
     \Rightarrow & \hat P\hat f = \hat f - \hat f \bk_\bx^\top (\bL + \gamma\bI_n)^{-1}\bell_{\bx_2} \\
     & \quad\;\,  = \by^\top \left(\bK + \lambda \bI_n\right)^{-1}\bk_\bx - \by^\top \left(\bK + \lambda \bI_n\right)^{-1} \bK (\bL + \gamma\bI_n)^{-1}\bell_{\bx_2} \\
     & \quad\;\, = \by^\top \!\left(\bK + \lambda \bI_n\right)^{-1}\!\left(\bk_\bx - \bK(\bL + \gamma \bI_n)^{-1}\bell_{\bx_2}\right) 
\end{align}

The kernel on $\cX_2$ is taken as $\ell = \kappa_{\theta_2}$. The CME regularisation weight $\gamma >0$ is tuned using a cross-validated grid search on an independently generated validation set.

\paragraph{\ding{227} $\cH_P$-KRR} We use the same base kernel as for KRR with again $\ell = \kappa_{\theta_2}$. We implement our estimator of the projected kernel $k_P$ is GPyTorch~\cite{gardner2018gpytorch}\footnote{which can be readily incorporated into GP regression pipelines.}. The kernel lengthscales and regularisation weights are tuned using a cross-validated grid search on an independently generated validation set.

\newpage
\subsection{Simulation example}

\paragraph{Data generating process}

Algorithm~\ref{alg:generate-sigma} outlines the procedure we use to generate a positive definite matrix $\Sigma$ that encodes independence between $X_2$ and $Y$.

\begin{algorithm}[h]
    \caption{Procedure to generate $\Sigma$}
    \label{alg:generate-sigma}
\begin{algorithmic}[1]
    \STATE {\bfseries Input:} $d_1 \geq 1$, $d_2 \geq 1$
    \STATE\texttt{\# Generate a $4 \times (d_1 + d_2 + 1)$ random matrix}
    \FOR{$i \in \{1, \ldots, d_1 + d_2 + 1\}$}
        \STATE $M_i \sim \cN(0, \bI_4)$
        \STATE $M_i \leftarrow M_i \,/\, \|M_i\|_2$
    \ENDFOR
    \STATE\texttt{\# Make $Y$ column orthogonal to all $X_2$ columns}
    \STATE $M_Y \leftarrow M_{d_1 + d_2 + 1}$
    \FOR{$i \in \{d_1 + 1, \ldots, d_1 + d_2\}$}
        \STATE $M_i \leftarrow M_i - (M_i^\top M_Y) M_Y$
    \ENDFOR
    \STATE $M \leftarrow \begin{bmatrix} M_1 \mid & \ldots & \mid M_{d_1 + d_2} \mid & \!\! M_Y\end{bmatrix}\in\RR^{4\times (d_1 + d_2 +1)}$
    \STATE $\Sigma \leftarrow M^\top M + 0.01\ast \bI_{d_1 + d_2 + 1}$
    \STATE \texttt{\# Normalise variances to 1}
    \STATE $\Lambda \leftarrow \operatorname{Diag}(\Sigma)$
    \STATE $\Sigma \leftarrow \Lambda^{-1/2} \Sigma \Lambda^{-1/2}$
    \STATE{\textbf{Return} $\Sigma$}
\end{algorithmic}
\end{algorithm}

\paragraph{Non-linear mappings}

The mappings $g_1$ and $g_2$ are applied to each component of the input vectors and are given by
\begin{align}
    g_1(u) & = u + 0.1\,\cos(2\pi u^2) \\
    g_2(u) & = u + 0.1\,\sin(2\pi u^2).
\end{align}

\paragraph{Statistical significance table}
\strut 

\begin{figure}[H]
    \centering
    \includegraphics[width=0.45\textwidth]{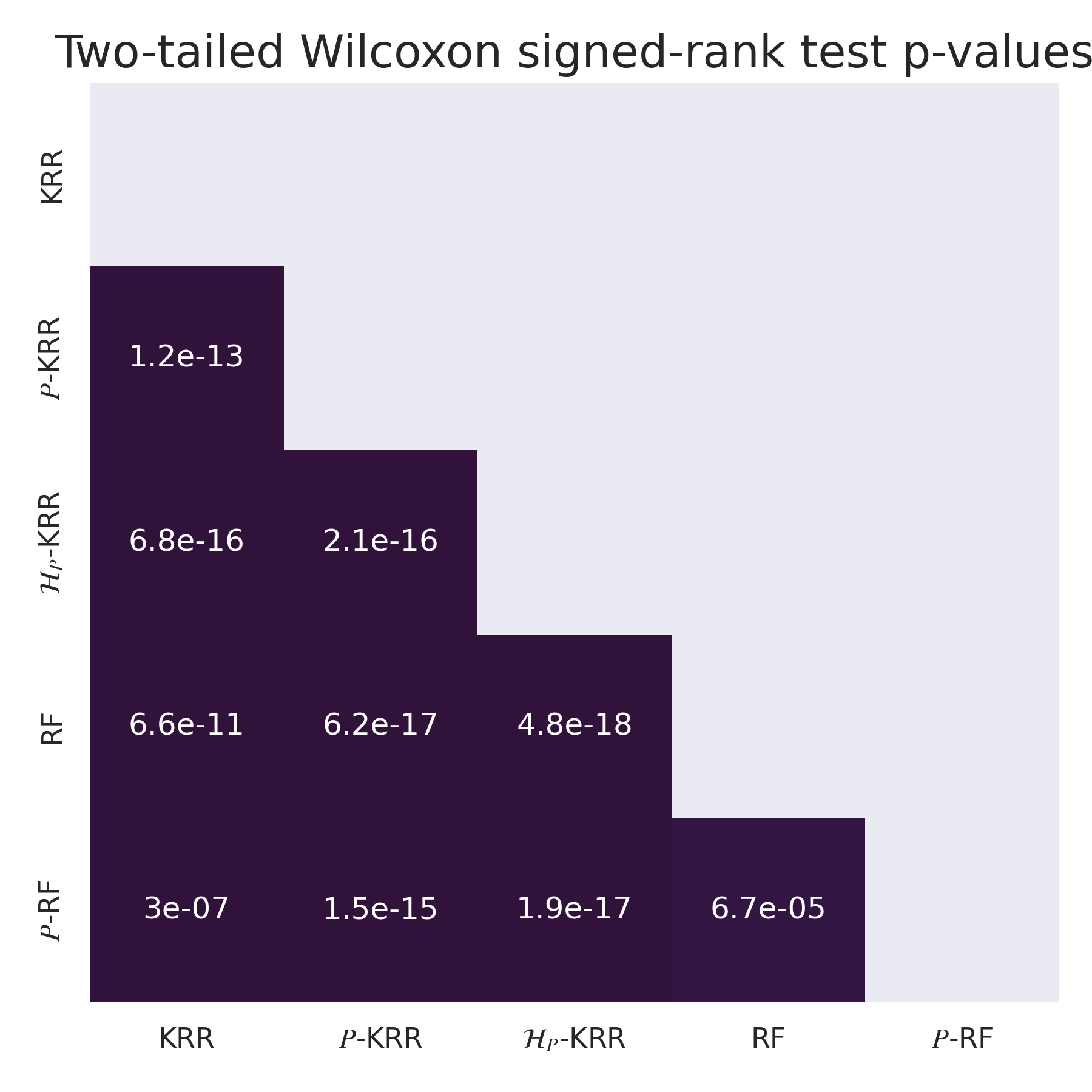}
    \caption{p-values from a two-tailed Wilcoxon signed-rank test between all pairs of methods for the
test MSE of the simulation example. The null hypothesis
is that scores samples come from the same distribution. We only present the lower triangular matrix
of the table for clarity of reading.}
\end{figure}

\newpage
\subsection{Aerosol radiative forcing}

\paragraph{Statistical significance table}
\strut

\begin{figure}[H]
    \centering
    \includegraphics[width=0.45\textwidth]{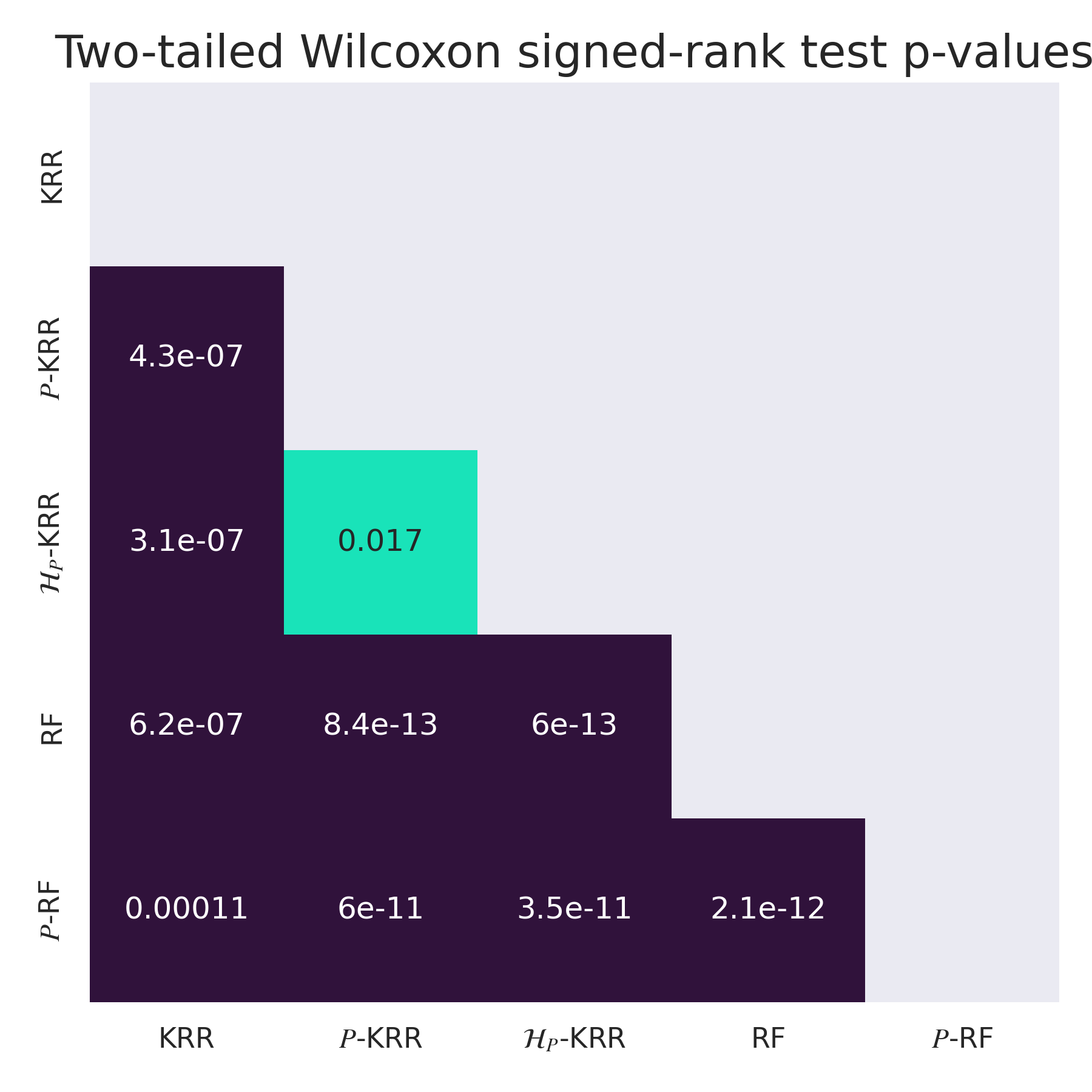}
    \caption{p-values from a two-tailed Wilcoxon signed-rank test between all pairs of methods for the
test \textbf{MSE} of the aerosol radiative forcing experiment. The null hypothesis
is that scores samples come from the same distribution. We only present the lower triangular matrix
of the table for clarity of reading.}
\end{figure}

\begin{figure}[H]
    \centering
    \includegraphics[width=0.45\textwidth]{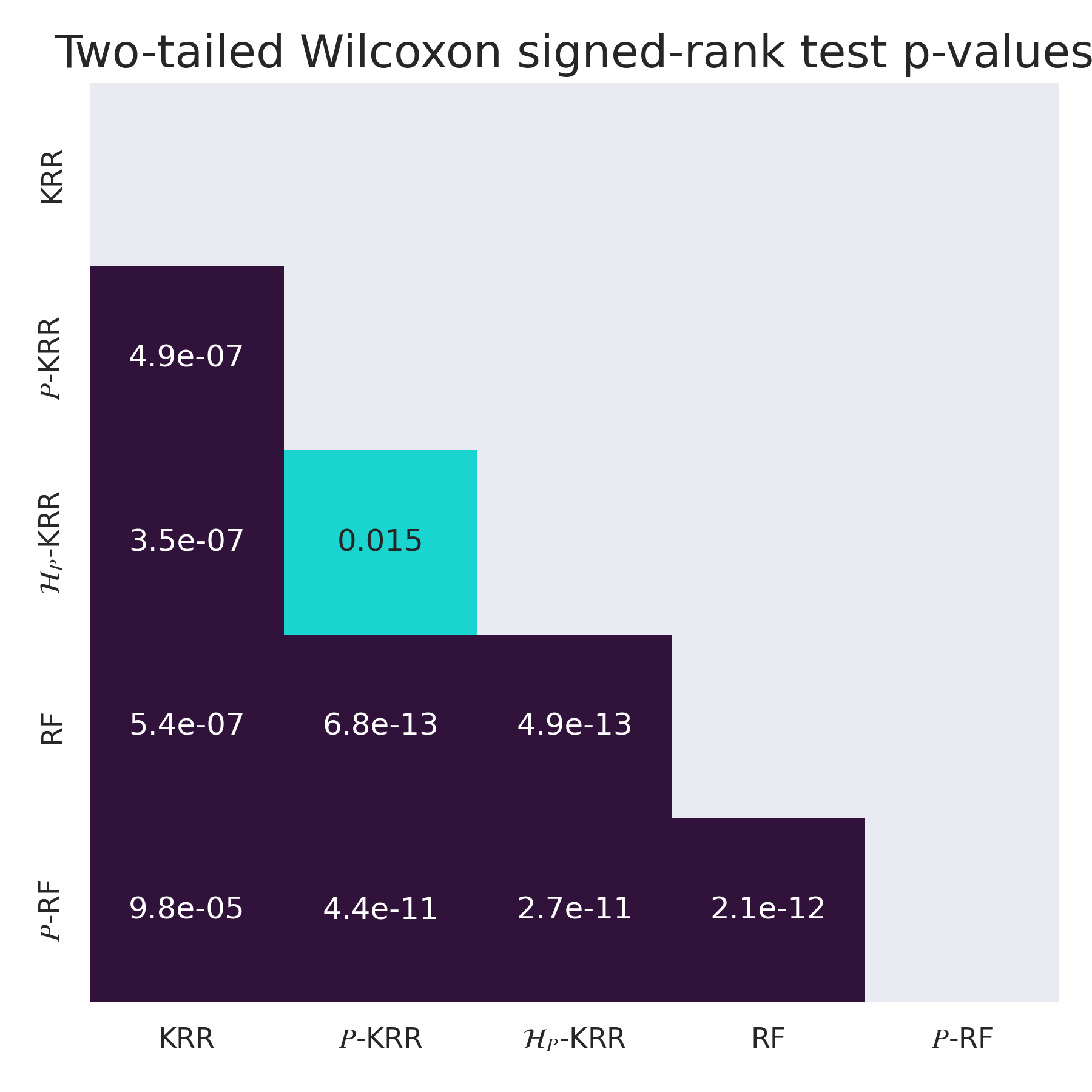}
    \caption{p-values from a two-tailed Wilcoxon signed-rank test between all pairs of methods for the
test \textbf{SNR} of the aerosol radiative forcing experiment. The null hypothesis
is that scores samples come from the same distribution. We only present the lower triangular matrix
of the table for clarity of reading.}
\end{figure}

\begin{figure}[H]
    \centering
    \includegraphics[width=0.45\textwidth]{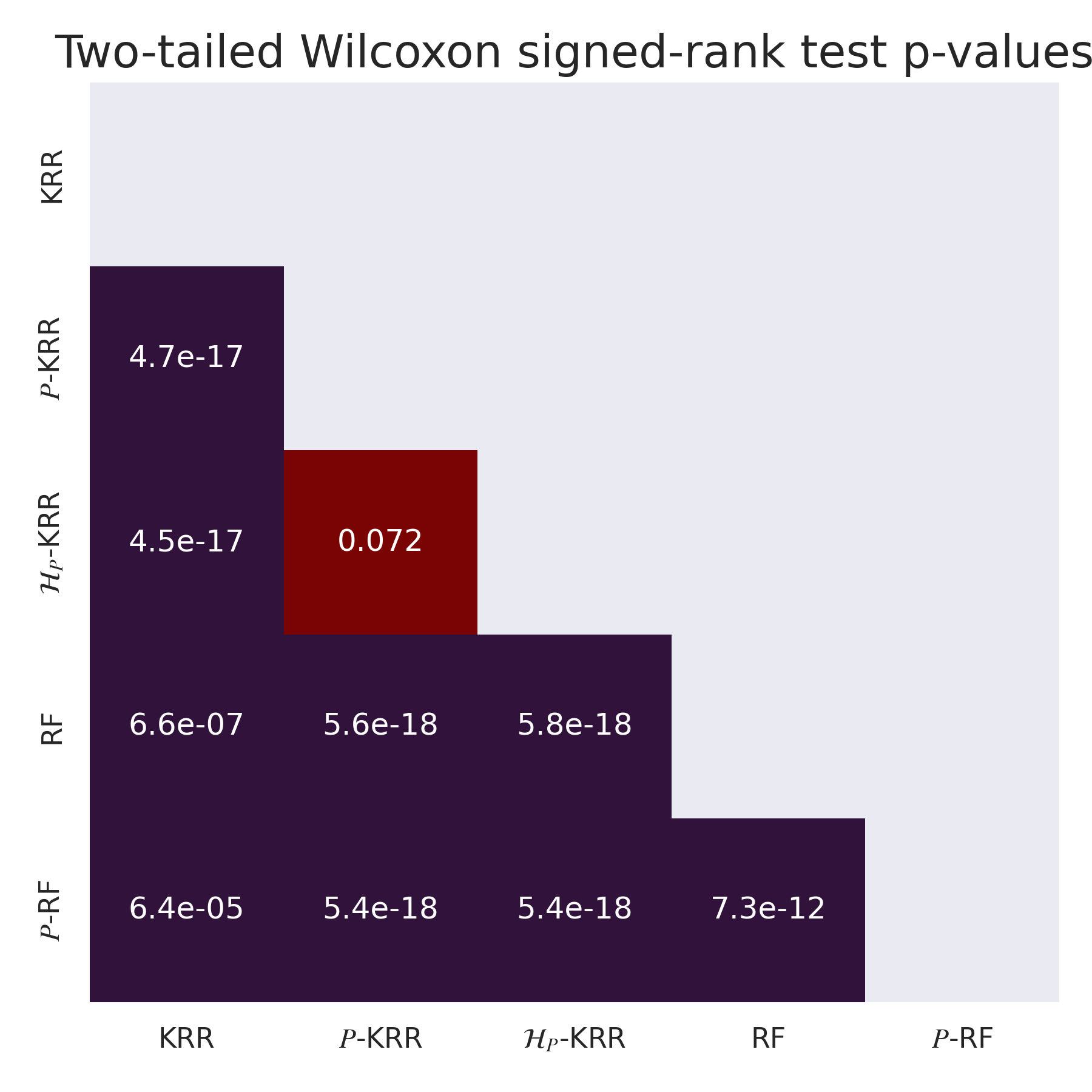}
    \caption{p-values from a two-tailed Wilcoxon signed-rank test between all pairs of methods for the
test \textbf{correlation} of the aerosol radiative forcing experiment. The null hypothesis
is that scores samples come from the same distribution. We only present the lower triangular matrix
of the table for clarity of reading.}
\end{figure}

\newpage
\section{Future direction}

\subsection{Extension to Gaussian processes}

\paragraph{Extension to Gaussian processes} The methodology presented can naturally be extended to the Bayesian counterpart of kernel ridge regression, Gaussian processes (GPs)~\cite{rasmussen2005gaussian}. One can either apply the projection operator $P : L^2(X)\to L^2(X)$ to the GP prior (or posterior), or use the projected kernel $k_P$ to specify the covariance function\footnote{Our implementation of $\hat k_P$ is available in GPyTorch~\cite{gardner2018gpytorch} and can be readily incorporated into GP regression pipelines.}.

However, such approach raises important questions from a theoretical perspective. If $f\sim\GP(0, k)$, the application of the $L^2(X)$ projection to $f$ will result in a linearly transformed GP $Pf\sim\GP(0, PkP^*)$~\cite{sarkka2011linear} and its draws will lie in the range of $P$. In contrast, since draws from a GP almost surely lie outside the RKHS associated with its covariance~\cite{kanagawa2018gaussian}, draws from $f \sim \operatorname{GP}(0, k_P)$ will almost surely lie outside $\cH_P$. It is therefore unclear whether these draws will lie in the range of the projection and satisfy the desired constraint for $f$. On the other hand, the posterior mean of the GP will always lie in $\cH_P$.

Furthermore, the projection is targeted at improving performance in mean square error. Because this metric is not necessarily adequate to evaluate GPs, it is unclear whether applying the projection would result in a performance improvement on more commonly used metrics for GPs such as maximum likelihood.

\end{document}